\documentclass[pdflatex,sn-mathphys-ay]{sn-jnl}

\usepackage{graphicx}%
\usepackage{multirow}%
\usepackage{amsmath,amssymb,amsfonts}%
\usepackage{amsthm}%
\usepackage{mathrsfs}%
\usepackage[title]{appendix}%
\usepackage{xcolor}%
\usepackage{textcomp}%
\usepackage{manyfoot}%
\usepackage{booktabs}%
\usepackage{algorithm}%
\usepackage{algorithmicx}%
\usepackage{algpseudocode}%
\usepackage{listings}%

\usepackage{xfrac}
\usepackage{times}

\theoremstyle{thmstyleone}%
\newtheorem{theorem}{Theorem}
\newtheorem{lemma}[theorem]{Lemma}%

\theoremstyle{thmstyletwo}%

\DeclareMathOperator{\E}{\mathbb{E}}

\DeclareMathOperator*{\Err}{Err}
\DeclareMathOperator{\tv}{TV}
\DeclareMathOperator{\kl}{D_{KL}}
\newcommand{\bth}{\boldsymbol{\theta}}

\begin{document}

\title[Overspecified MDA]{Overspecified Mixture Discriminant Analysis: Exponential Convergence, Statistical Guarantees, and Remote Sensing Applications}




\author[1,2]{\fnm{Arman} \sur{Bolatov}}\email{arman.bolatov@mbzuai.ac.ae}

\author[3]{\fnm{Alan} \sur{Legg}}\email{leggar01@pfw.edu}

\author[4]{\fnm{Igor} \sur{Melnykov}}\email{imelnyko@d.umn.edu}

\author[5]{\fnm{Amantay} \sur{Nurlanuly}}\email{amantay.nurlanuly@nu.edu.kz}

\author[5]{\fnm{Maxat} \sur{Tezekbayev}}\email{maxat.tezekbayev@alumni.nu.edu.kz}

\author*[2,3]{\fnm{Zhenisbek} \sur{Assylbekov}}\email{zassylbe@pfw.edu}

\affil[1]{\orgdiv{Machine Learning Department}, \orgname{Mohamed bin Zayed University of Artificial Intelligence}, \orgaddress{\street{Masdar City}, \city{Abu Dhabi}, \postcode{00000}, \country{UAE}}}

\affil[2]{\orgdiv{Remote Sensing Department}, \orgname{National Center of Space Research and Technology}, \orgaddress{\street{15 Shevchenko St.}, \city{Almaty}, \postcode{050010}, \country{Kazakhstan}}}

\affil[3]{\orgdiv{Department of Mathematical Sciences}, \orgname{Purdue University Fort Wayne}, \orgaddress{\street{2101 East Coliseum Boulevard}, \city{Fort Wayne}, \postcode{46805}, \state{IN}, \country{USA}}}

\affil[4]{\orgdiv{Department of Mathematics and Statistics}, \orgname{University of Minnesota Duluth}, \orgaddress{\street{1049 University Drive}, \city{Duluth}, \postcode{55812}, \state{MN}, \country{USA}}}

\affil[5]{\orgdiv{Department of Mathematics}, \orgname{Nazarbayev University}, \orgaddress{\street{53 Kabanbay Batyr ave.}, \city{Astana}, \postcode{010000}, \country{Kazakhstan}}}


\abstract{This study explores the classification error of Mixture Discriminant Analysis (MDA) in scenarios where the number of mixture components exceeds those present in the actual data distribution, a condition known as overspecification. We use a two-component Gaussian mixture model within each class to fit data generated from a single Gaussian, analyzing both the algorithmic convergence of the Expectation-Maximization (EM) algorithm and the statistical classification error. We demonstrate that, with suitable initialization, the EM algorithm converges exponentially fast to the Bayes risk at the population level. Further, we extend our results to finite samples, showing that the classification error converges to Bayes risk with a rate $n^{-1/2}$ under mild conditions on the initial parameter estimates and sample size. This work provides a rigorous theoretical framework for understanding the performance of overspecified MDA, which is often used empirically in complex data settings, such as image and text classification. To validate our theory, we conduct experiments on remote sensing datasets.}


\maketitle

\section{Introduction}

Linear discriminant analysis (LDA), a fundamental classification approach, dates back to the works of  \cite{fisher1936use} and \cite{06a3d571-d8a5-3c28-8e78-7b1e47c0d070}. It assumes that in each class the features are generated from a multivariate Gaussian distribution with its own mean, and all classes share the same covariance matrix. For specifics, let us consider the simplest LDA setup, in which the classes are balanced and the class-conditional distributions are spherical $d$-variate Gaussians $\mathcal{N}(\pm\boldsymbol{\mu},\mathbf{I})$. The celebrated Neyman-Pearson lemma \citep{doi:10.1098/rsta.1933.0009} implies that the Bayes classifier in this case is
$\mathbf{x}\mapsto\mathrm{sgn}\left[\log\frac{\phi(\mathbf{x}-\boldsymbol{\mu})}{\phi(\mathbf{x}+\boldsymbol{\mu})}\right]$ and its misclassification error (Bayes risk) is 
$\Phi(-\|\boldsymbol{\mu}\|)$ (see Section 1.2.1 in \cite{Wainwright_2019}), where $\phi(\mathbf{x})$ is the probability density function of $\mathcal{N}(\mathbf{0},\mathbf{I})$ and $\Phi(x)$ is the cumulative distribution function of $\mathcal{N}(0,1)$. 

In practice, the true value of the parameter $\boldsymbol{\mu}$ is unknown, and must be estimated from a finite sample (of size $n$). Using standard  concentration inequalities, it is easy to verify that replacing $\boldsymbol{\mu}$ with its  maximum likelihood estimate $\hat{\boldsymbol{\mu}}$  we can get a finite-sample error bound for the LDA classifier in the form $\Phi(-\|\boldsymbol{\mu}\|)+{O}(\sqrt{d/n})$ 
with high probability over the choice of a training sample.

However, the main problem with LDA when applied to natural data—such as images or texts—is that such data typically consists of subpopulations. For example, consider the well-known MNIST dataset \citep{DBLP:journals/spm/Deng12}, which consists of handwritten digit examples, where the task is to recognize the digit from the image. If we take one class, say `three', there are examples of a typically written `3', as shown in the top row of the right part of Figure~\ref{fig:mnist}. These correspond to the major blue cluster in the center of the left part of Figure~\ref{fig:mnist}, which shows projections of images from the class `three' on a 2D plane with UMAP \citep{McInnes2018}. However, there are also examples of less typically written `3's that correspond to the other clusters.
\begin{figure}[htbp]
\centering
\includegraphics[width=340pt]{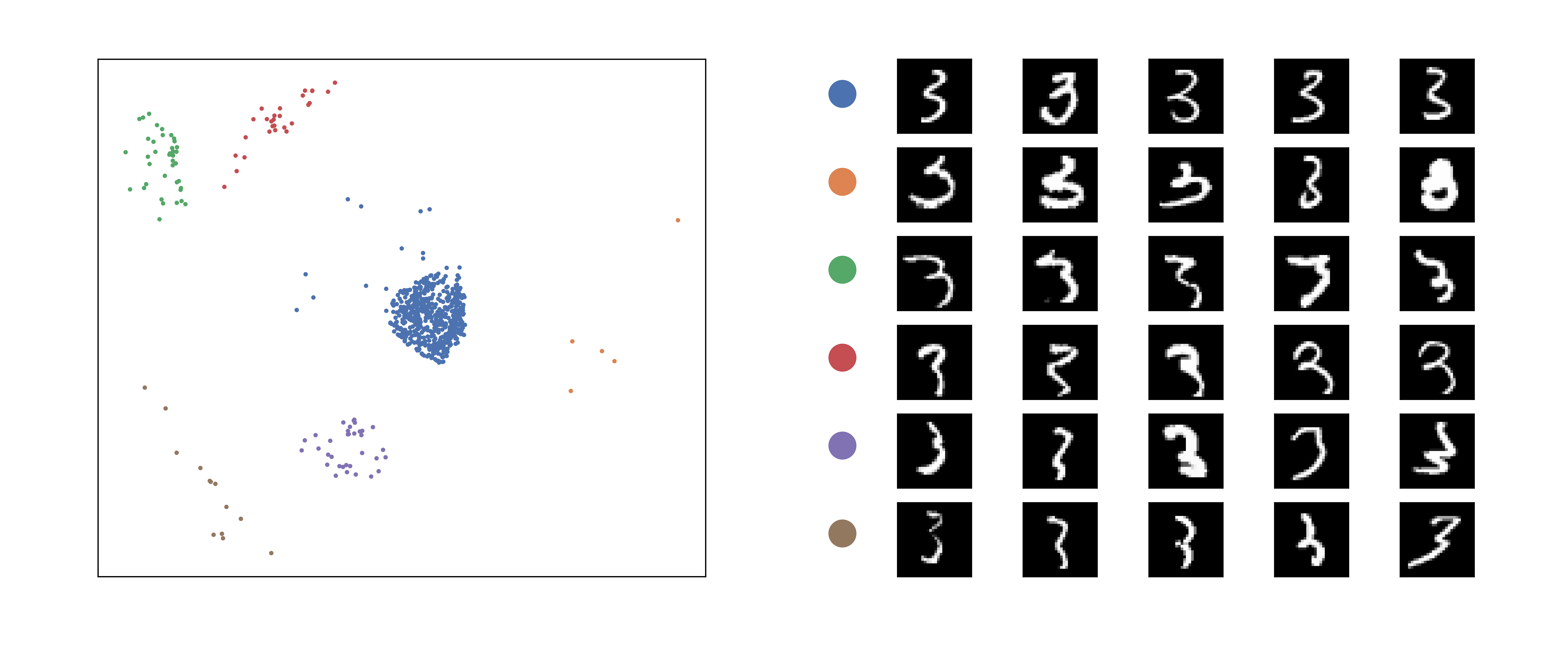}
\caption{Projections of MNIST \citep{DBLP:journals/spm/Deng12} images of `3' on a 2D plane with UMAP \citep{McInnes2018} (left) and examples of images from each cluster (right). Clustering was performed with HDBSCAN \citep{DBLP:conf/pakdd/CampelloMS13}.}
\label{fig:mnist}
\end{figure}
Thus, we see that the distribution even within one class might consist of several clusters, which a unimodal Gaussian distribution cannot model.

Fortunately, in this case we can replace the Gaussians in each class with mixtures of Gaussians \citep{https://doi.org/10.1111/j.2517-6161.1996.tb02073.x}. This modification of LDA was dubbed Mixture Discriminant Analysis (MDA) and is widely used due to its greater flexibility.\footnote{As of this writing, the work of \cite{https://doi.org/10.1111/j.2517-6161.1996.tb02073.x} has been cited more than 1000 times according to Google Scholar.} Despite this, to the best of our knowledge, all earlier studies on MDA were purely empirical offering no rigorous statistical or algorithmic guarantees for MDA. 

An especially important case is when the true number of mixture components is \emph{unknown}, and in practice, this is usually the case. One way to alleviate this problem is to estimate the required number of components before fitting a mixture \citep{6360018}. An alternative is  \emph{over-specification} \citep{4e9d0d9d-0f59-372f-8783-3bcaf4dfc2ec}, that is, when during training, obviously more components are used than actually in the data distribution, in the hope that the learning algorithm---usually expectation-maximization \cite[EM]{https://doi.org/10.1111/j.2517-6161.1977.tb01600.x}---will fit a mixture in such a way that it will not differ noticeably from a true mixture that has fewer components. 

In this work, we consider the  scenario of over-specification, when for each class the data is generated from one Gaussian, but we fit an \emph{unbalanced} Gaussian mixture of two components by the EM algorithm. We are interested in both the algorithmic convergence of the EM and the statistical classification error. To distinguish between these, we initially consider the problem at the population level by assuming access to an infinitely large sample. On this front, we demonstrate that the classification error converges to the Bayes risk exponentially fast. Namely, under suitable initialization the EM algorithm takes $O(\log(1/\epsilon))$ steps to reduce the classification error to $\Phi(-\|\boldsymbol{\mu}\|)+\epsilon$. Encouraged by this result, we move on to study statistical convergence when learning an MDA classifier from a finite sample. Here we are able to show that the EM algorithm produces an MDA classifier whose classification error is $\Phi(-\|\boldsymbol{\mu}\|)+O(\sqrt{d/n})$ with high probability over the choice of a training sample of size $n$, and for this it suffices to take $O(\log(n/d))$ steps. 

Recently, \citet{legg2025convergence} provided MDA error bounds for the \emph{balanced} overspecified two-component Gaussian mixture: they showed that the EM algorithm requires $\tilde{O}(\sqrt{n/d})$ iterations to learn an MDA classifier with error rate $\Phi(-\lVert\mu\rVert)+O\!\left(\sqrt{d/n}\right)$. Our contribution addresses the \emph{unbalanced} two-component case. We establish exponential population contraction of the misclassification error to the Bayes risk and achieve the same error rate, $\Phi(-\lVert\mu\rVert)+O\!\left(\sqrt{d/n}\right)$, but in only $O(\log(n/d))$ EM iterations. In what follows, we state our results formally.

\paragraph{Notation.}
Bold-faced lowercase letters ($\mathbf{x}$) denote vectors in $\mathbb{R}^d$, bold-faced uppercase letters ($\mathbf{A}$, $\mathbf{X}$) denote matrices and random vectors, 
regular lowercase letters ($x$) denote scalars. 
The Euclidean norm is denoted by $\|\mathbf{x}\|:=\sqrt{\mathbf{x}^\top\mathbf{x}}$. 

`P.d.f.' stands for `probability density function', and `c.d.f.' stands for `cumulative distribution function'. 
The p.d.f. of $\mathbf{Z}\sim\mathcal{N}(\mathbf{0},\mathbf{I})$, where $\mathbf{I}$ is a $d\times d$ identity matrix, is denoted by $\phi(\mathbf{z})$. The p.d.f. and c.d.f. of $Z\sim\mathcal{N}(0,1)$ are denoted by $\phi(z)$ and $\Phi(z)$ respectively. 


Given $f:\,\mathbb{R}\to\mathbb{R}$ and $g:\,\mathbb{R}\to\mathbb{R}_+$, we write $f\precsim g$  if there exist $x_0,c\in\mathbb{R}_+$ such that for all $x>x_0$ we have $|f(x)|\le c g(x)$. When $f:\mathbb{R}\to\mathbb{R}_+$, we write $f\asymp g$  if $f\precsim g$ and $g\precsim f$. We use $c$, $c_1$, $c_2$, etc. to denote some
universal constants (which might change in value each time they appear).

\section{Related Work}

The convergence analysis of the Expectation-Maximization (EM) algorithm for Gaussian mixture models has undergone significant development in recent years. \citet{DBLP:journals/corr/BalakrishnanWY14} proposed a framework for characterizing the algorithm’s convergence region with respect to distribution parameters, distinguishing between population-level analysis and the sample-based implementation used in practice. Their results, focused on the correctly specified setting with $k=2$ components, covered both balanced and unbalanced mixtures.

A central issue within this framework is the initialization strategy required for convergence to the global optimum. \citet{klusowski2016statistical} extended the local convergence region identified by \citet{DBLP:journals/corr/BalakrishnanWY14} for the two-component case, while \citet{zhao2020statistical} investigated initialization effects for an arbitrary number of well-separated components. \citet{daskalakis2017ten} proved global convergence in the two-component case with symmetric means about the origin. For $k$ well-separated components, \citet{segol2021improved} established convergence even when initialization occurs near the midpoint between two clusters, improved estimation error bounds, and obtained similar results for Gradient EM. The convergence rate and local contraction properties of Gradient EM for arbitrary $k$ were further analyzed by \citet{yan2017convergence}.

Model misspecification has also attracted considerable attention. \citet{dwivedi2018theoretical} analyzed an underspecified setting where a two-component Gaussian mixture is fitted to data from a three-component mixture, quantifying the resulting bias and examining initialization effects. The benefits of overspecified models have been documented in \citet{dwivedi2020singularity}, \citet{DBLP:conf/aistats/DwivediHKWJ020}, \citet{chen2024local}, and related work. In particular, \citet{dwivedi2020singularity} and \citet{DBLP:conf/aistats/DwivediHKWJ020} studied the case of fitting two Gaussians to data from a single Gaussian, comparing balanced and unbalanced mixtures. In the sample-based EM, they showed that the unbalanced case attains an $O(1/\sqrt{n})$ statistical rate—substantially faster than the $O(n^{-1/4})$ rate observed in the balanced case—under both known and estimated isotropic covariances, and also enjoys exponentially faster algorithmic convergence.

In the Bayesian context, \citet{rousseau} showed that overspecified mixtures tend to yield highly imbalanced component weights, with some components effectively vanishing; removing these can improve the model. Regarding spurious components, \citet{chen2024local} demonstrated that every local minimum of the negative log-likelihood, including spurious ones, retains structural information about component means. They emphasized the advantages of overspecification over underspecification, contrasting “many-fit-one’’ with “one-fit-many.’’ \citet{dasgupta2013two} further advocated deliberate overspecification, suggesting $\frac{\log(k)}{w_{\min}}$ initial clusters—where $w_{\min}$ is the smallest mixture weight—to accelerate convergence.

While the literature has largely focused on convergence in parameter space, fewer works assess fit quality using the Kullback–Leibler (KL) divergence. \citet{ghosal2001} derived a statistical rate of $(\log n)^\kappa/\sqrt{n}$ in Hellinger distance, implying a KL lower bound of $(\log n)^{2\kappa}/n$ in the well-specified setting, without addressing algorithmic convergence. \citet{dwivedi2018theoretical} employed KL divergence for underspecified mixtures, but, to our knowledge, the first investigation of KL bounds for overspecified mixtures appeared in \citet{xu2024toward}, who analyzed the population Gradient EM for $k$-component mixtures with known variances. Our work extends these results to both population and sample-based EM with learned variances in the two-component case, and establishes a sharper algorithmic convergence rate in KL distance than that obtained by \citet{xu2024toward}.

\section{Main results} Let $Y$ be a  Rademacher random variable, i.e. $\Pr[Y=-1]=\Pr[Y=+1]=1/2$, and let $\mathbf{X}$ given $Y=y$ have a Gaussian distribution $\mathcal{N}( y\boldsymbol{\mu},\mathbf{I})$. These conditions define the distribution for the pair $(\mathbf{X},Y)$, which we denote by $\mathcal{D}$. Suppose that we approximate each of the Gaussians $\mathcal{N}(\pm\boldsymbol{\mu},\mathbf{I})$ by an unbalanced two-component location-scale Gaussian mixture
\begin{equation}
\mathcal{G}(\pm\boldsymbol{\mu},\bth,\sigma^2)= (1-p)\cdot\mathcal{N}(\pm\boldsymbol{\mu}-\bth,\sigma^2\mathbf{I})+p\cdot\mathcal{N}(\pm\boldsymbol{\mu}+\bth,\sigma^2\mathbf{I}),\quad p\ne\sfrac12.\label{eq:gmm_class}
\end{equation}
The choice of an unbalanced mixture is motivated by the fact that in this case the EM algorithm converges much faster than in the case of a balanced mixture \citep{dwivedi2020singularity}. 

The parameters $\boldsymbol{\mu},\bth,\sigma^2$  in \eqref{eq:gmm_class} are estimated from a finite sample $\{(\mathbf{X}_i,Y_i)\}_{i=1}^n$ $\stackrel{\text{iid}}{\sim}\mathcal{D}$. Due to symmetry, $\{\mathbf{X}_i Y_i\}_{i=1}^n\stackrel{\text{iid}}{\sim}\mathcal{N}(\boldsymbol{\mu},\mathbf{I})$, and thus for parameters estimation it is sufficient to fit a single mixture $\mathcal{G}(\boldsymbol{\mu},\bth,\sigma^2)$  to the sample $\{\mathbf{X}_iY_i\}_{i=1}^n$. Specifically, the maximum likelihood estimate (MLE) of  $\boldsymbol{\mu}$ is given by
\begin{equation}
\hat{\boldsymbol\mu}=\frac1n\sum_{i=1}^n\mathbf{X}_i\cdot Y_i,\label{eq:mu_mle}    
\end{equation}
while the MLE of $(\bth,\sigma^2)$ can be obtained by the EM algorithm which produces a sequence of parameter estimates $(\hat\bth_t,\hat\sigma^2_t)$.  Hence we are dealing with a sequence of the corresponding MDA classifiers
\begin{equation}
\hat{h}_t(\mathbf{x})=\begin{cases}
    +1\quad&\text{if }f(\mathbf{x};\hat{\boldsymbol{\mu}},\hat\bth_t,\hat\sigma^2_t)>f(\mathbf{x};-\hat{\boldsymbol{\mu}},\hat\bth_t,\hat\sigma^2_t)\\
    -1&\text{otherwise},\label{eq:mda_clf}
\end{cases}    
\end{equation}
where $f(\mathbf{x};\boldsymbol{\mu},\bth,\sigma^2)$ is the density of $\mathcal{G}(\boldsymbol{\mu},\bth,\sigma^2)$.

By the law of large numbers, $\hat{\boldsymbol{\mu}}\stackrel{P}{\to}\boldsymbol{\mu}$ as $n\to\infty$. Similarly, let $(\bth_t,\sigma^2_t)$ be the population versions of $(\hat\bth_t,\hat\sigma^2_t)$, i.e. $\hat\bth_t\stackrel{P}{\to}\bth_t$, $\hat\sigma^2_t\stackrel{P}{\to}\sigma^2_t$ as $n\to\infty$, and consider the sequence of the corresponding population-level MDA classifiers
\begin{equation}
{h}_t(\mathbf{x})=\begin{cases}
    +1\quad&\text{if }f(\mathbf{x};{\boldsymbol{\mu}},\bth_t,\sigma^2_t)>f(\mathbf{x};-{\boldsymbol{\mu}},\bth_t,\sigma^2_t)\\
    -1&\text{otherwise},\label{eq:mda_clf_pop}
\end{cases}
\end{equation}

Our first result shows that if we set aside sample complexity for a moment and give the EM algorithm access to a sample of infinite size, we get an \emph{exponential} (in $t$) convergence  of the MDA classification error rate to the Bayes risk.
\begin{theorem}\label{thm:mda_pop}
    For any starting point $\bth_0$ such that 
    $
    \|\boldsymbol{\theta}_0\|<\min\left[\sqrt{d\cdot\frac{2+q-\sqrt{8q+q^2}}{2}},\frac{1}{\sqrt2+\frac1{\sqrt2d}}\right]
    $
    where $q:=1-\frac{(2p-1)^2}{2}\in(0,1)$, the population EM algorithm produces a sequence $(\bth_t,\sigma^2_t)$ such that the MDA classifier \eqref{eq:mda_clf_pop} satisfies
    $$
    \Pr_{\mathbf{X},Y\sim\mathcal{D}}\left[h_T(\mathbf{X})\ne Y\right]\le\Phi(-\|\boldsymbol{\mu}\|)+\epsilon,
    $$
    for $T\ge c\log(1/\epsilon)$ where $c>0$ is a constant.
\end{theorem}

Such fast algorithmic convergence of the EM gives reason to hope that when moving to a sample of finite size, the statistical error will not accumulate quickly. Indeed, our second result shows that the statistical convergence of misclassification error rate to Bayes risk is of order $O(\sqrt{d/n})$.
\begin{theorem}\label{thm:mda}
    Fix $\delta\in(0,1)$, $\alpha\in(0,1/2)$, and let $n\succsim \log^{\frac1{2\alpha}}\left(1/\delta\right)$. Then with any starting point $\bth_0$ such that $\|\boldsymbol{\theta}_0\|<\min\left[\sqrt{d\cdot\frac{2+q-\sqrt{8q+q^2}}{2}},\frac{1}{\sqrt2+\frac1{\sqrt2d}}\right]$ where $q:=1-\frac{(2p-1)^2}{2}$, the EM algorithm produces a sequence of parameter estimates $(\hat\bth_t,\hat\sigma^2_t)$ such that the MDA classifier \eqref{eq:mda_clf} satisfies
    \begin{equation}
        \Pr_{\mathbf{X},Y\sim\mathcal{D}}[\hat{h}_T(\mathbf{X})\ne Y]\le \Phi\left(-\|\boldsymbol{\mu}\|\right)+c_1\|\bth_0\|\sqrt{\frac{\log(1/\delta)}{n}}.\label{eq:mda_clf_err_sample}
    \end{equation}
    for $T\ge c_2\log\frac{n}{\log(1/\delta)}$ with probability at least $1-\delta$.
\end{theorem}


The proof of both results is based on the fact that if we can well approximate the underlying class-conditional distributions (in the sense of Total Variation Distance, TVD), then the classification error will be low (Lemma~\ref{lem:excess}). Moreover, due to the well-known Pinsker inequality, for convergence in the TVD metric, convergence in the KL distance is sufficient. That is, everything reduces to studying the convergence of the sequence of mixtures $\mathcal{G}(\boldsymbol{\mu},\bth_t,\sigma^2_t)$ (or $\mathcal{G}(\hat{\boldsymbol{\mu} },\hat\bth_t,\hat\sigma^2_t)$) to the true underlying distribution $\mathcal{N}(\boldsymbol{\mu},\mathbf{I})$ in the KL metric using the EM algorithm, and such a study is the main technical contribution of this work (Section~\ref{sec:kl_dist_conv}).

The remainder of the paper is organized as follows. In Section~\ref{sec:pop_level}, we study the convergence in KL distance in the infinite sample size regime. It turns out that for a sequence of EM iterates $(\bth_t, \sigma^2_t)$, the current value of the variance $\sigma^2_t$ is a function of the norm $\|\bth_t\|^2$, and thus the sequence itself lies on a $d$-dimensional hypersurface. It is noteworthy that on this hypersurface, the expected negative log-likelihood is a \emph{convex radial} function in the vicinity of the optimal point, and it satisfies a so-called \emph{Polyak-Lojasiewicz} inequality. To the best of our knowledge, in the context of the considered setup, we are the first to make these observations and they greatly simplify the study of the convergence of the KL distance between the fitted and true distributions.  Next, in Section~\ref{sec:finite_sample}, we extend our result to the finite sample case. Here we rely on the convergence of the parameters themselves to optimal ones. Our result for the finite sample case is novel as well, since in the work of \cite{dwivedi2020singularity}, a similar result was obtained for the case when the variance parameter is known and is not learned from the data, while the work by \cite{DBLP:conf/aistats/DwivediHKWJ020} considers the case where variance is learned (along with location parameter), but a mixture is balanced.  In Section~\ref{sec:main_proof} we provide proofs of Theorems~\ref{thm:mda_pop}~and~\ref{thm:mda}. Proofs not included in the main body of the paper are given in Appendix~\ref{app:proofs}. Finally, in Section~\ref{sec:experiments}, we present our experimental results on remote sensing data.

\section{Convergence in KL distance.}\label{sec:kl_dist_conv} 

Formally, let  $\mathbf{Z}_1,\ldots,\mathbf{Z}_n$ be a random sample from a multivariate Gaussian distribution $\mathcal{N}(\mathbf{0},\mathbf{I})$ with mean vector $\mathbf{0}\in\mathbb{R}^d$ and identity covariance matrix $\mathbf{I}\in\mathbb{R}^{d\times d}$. Suppose that we want to fit an unbalanced two-component location-scale Gaussian mixture 
\begin{equation}
\mathcal{G}(\mathbf{0},\bth,\sigma^2)=(1-p)\cdot\mathcal{N}(-\bth,\sigma^2\mathbf{I})+p\cdot\mathcal{N}(\bth,\sigma^2\mathbf{I})\label{eq:gmm}    
\end{equation}
to this sample, assuming $p$ is fixed. In this section, with a slight abuse of notation, we will write $\mathcal{G}(\bth,\sigma^2)$ instead of $\mathcal{G}(\mathbf{0},\bth,\sigma^2)$. 

Let $f(\mathbf{x};\bth,\sigma^2)$ be the probability density function of the mixture \eqref{eq:gmm}. Then the maximum likelihood estimator (MLE) of $(\bth,\sigma^2)$ is given by
\begin{equation}
(\hat{\bth},\hat\sigma^2)\in\arg\max_{(\bth,\sigma^2)}\frac1n\sum_{i=1}^n\log f(\mathbf{Z}_i;\bth,\sigma^2).\label{eq:mle}
\end{equation}
There is no closed-form expression for $(\hat\bth,\hat\sigma^2)$. In practice, the optimization problem \eqref{eq:mle} is solved by an iterative method such as the EM algorithm \citep{https://doi.org/10.1111/j.2517-6161.1977.tb01600.x}. Note that the objective function (log-likelihood) in \eqref{eq:mle} is not concave, and accordingly, there is no guarantee that the iterative method will converge to the global optimum.

\subsection{Population-Level Analysis} \label{sec:pop_level}
First, we will look at the behavior of the so-called \emph{Population EM}, which is a theoretical construct that allows us to focus on algorithmic complexity without being distracted by sample complexity. Population EM assumes access to the data-generating distribution $\mathcal{N}(\mathbf{0},\mathbf{I})$, and instead of the sample-based log-likelihood in \eqref{eq:mle} it optimizes the population log-likelihood
\begin{equation}
\mathcal{L}(\bth,\sigma^2):=\E_{\mathbf{Z}\sim\mathcal{N}(\mathbf{0},\mathbf{I})}[\log f(\mathbf{Z};\bth,\sigma^2)].\label{eq:e_log_l}    
\end{equation}
This is done by iteratively applying the following two steps:
\begin{itemize}
    \item \emph{Expectation step} uses the current estimate $(\bth_t,\sigma^2_t)$ to compute the function
    \begin{multline*}
        Q(\bth,\sigma^2;\bth_t,\sigma^2_t):=\E_{\mathbf{Z}\sim\mathcal{N}(\mathbf{0},\mathbf{I})}\left[w(\mathbf{Z};\bth_t,\sigma^2_t)\log\left(p\phi\left(\frac{\mathbf{Z}-\bth}{\sigma}\right)\right)\right.\\\left.
        +(1-w(\mathbf{Z};\bth_t,\sigma^2_t))\log\left((1-p)\phi\left(\frac{\mathbf{Z}+\bth}{\sigma}\right)\right)\right],
    \end{multline*}
    where $w(\mathbf{z};\bth_t,\sigma^2_t)=\frac{p}{p+(1-p)\cdot\exp\left(-\frac{2\bth_t^\top\mathbf{z}}{\sigma^2_t}\right)}$.
    \item \emph{Maximization step} solves the optimization problem
    $$
    (\bth_{t+1},\sigma^2_{t+1})\in\arg\max_{(\bth,\sigma^2)}Q(\bth,\sigma^2;\bth_t,\sigma^2_t).
    $$
\end{itemize}
One can show that in this special case of fitting the mixture \eqref{eq:gmm} to $\mathcal{N}(\mathbf{0},\mathbf{I})$ by the Population EM, the recurrence relations governing the evolution of the parameters can be written down explicitly (Section~\ref{app:EM_upd_proof}):

    \begin{align}
        \bth_{t+1}&=\E_{\mathbf{Z}\sim\mathcal{N}(\mathbf{0},\mathbf{I})}\left[t_p\left(\frac{\bth_t^\top\mathbf{Z}}{1-\frac{\|\bth_t\|^2}{d}}\right)\mathbf{Z}\right],\label{eq:theta_upd}\\
        \sigma^2_{t+1}&=1-\frac{\|\bth_{t+1}\|^2}{d},\label{eq:sigma_sq_upd}
    \end{align}
    where $t_p(x):=\frac{p\cdot e^{x}-(1-p)\cdot e^{-x}}{p\cdot e^{x}+(1-p)\cdot e^{-x}}$. As we see, the iterates $(\bth_t,\sigma^2_t)$ lie on the hypersurface 
\begin{equation}
\mathcal{S}:=\left\{(\bth,\sigma^2)\in\mathbb{R}^{d+1}:\,\,\sigma^2=1-{\|\bth\|^2}/{d}\right\},\label{eq:hypersurf}
\end{equation}
and the value of $\sigma^2_t$ is completely determined by the norm $\|\bth_t\|$. In this regard, we are interested in the form of the population log-likelihood \eqref{eq:e_log_l} when $(\bth,\sigma^2)\in\mathcal{S}$. First of all, we notice that in this case, the population log-likelihood depends on $\bth$ only through its norm.

\paragraph{Radial form of the risk function.} 

Consider the function $L:\mathbb{R}^d\to\mathbb{R}$ defined as
\begin{equation}
    L(\bth):=-\mathcal{L}(\bth,1-\|\bth\|^2/d).\label{eq:pop_risk}    
\end{equation}
Then $L(\bth)$ is a \emph{radial} function of $\bth\in\mathbb{R}^d$ (see Appendix~\ref{app:ell_radial}). Specifically, let $\theta:=\|\bth\|$, and define 
\begin{equation}
\ell(\theta):=\frac{d}2\log(2\pi(1-\theta^2/d))+\frac{d+\theta^2}{2(1-\theta^2/d)}-\E_{Z\sim\mathcal{N}(0,1)}\left[\log\left(c_p\left(\frac{\theta Z}{1-\theta^2/d}\right)\right)\right].\label{eq:ell}
\end{equation}
Then we have 
\begin{equation}\ell(\theta)=L(\bth)\qquad\text{i.e. the risk depends on $\bth$ only via $\theta$}.\label{eq:ell_rad}
\end{equation}

We put a minus sign in front of the population log-likelihood to make $L(\cdot)$ have the meaning of a risk function. In this way, maximizing $\mathcal{L}(\bth,\sigma^2)$  on $\mathcal{S}$ is equivalent to minimizing $L(\bth)$ on $\mathbb{R}^d$, which in turn reduces to minimizing $\ell(\theta)$ on $\mathbb{R}$.

\paragraph{Population EM operator in one dimension.}
According to the EM updates \eqref{eq:theta_upd}--\eqref{eq:sigma_sq_upd}, it is sufficient to focus on analyzing the evolution of $\bth_t$ which is governed by the \emph{population EM operator} 
\begin{equation}
M(\bth) = \E_{\mathbf{Z}\sim\mathcal{N}(\mathbf{0},\mathbf{I})}\left[t_p\left(\frac{\bth^\top\mathbf{Z}}{1-\sfrac{\|\bth\|^2}{d}}\right)\mathbf{Z}\right].\label{eq:pop_em_op}
\end{equation}
Thus, for the EM sequence $\bth_{t}$, we have $\bth_{t+1} = M(\bth_t)$. However, as we see from \eqref{eq:pop_risk}, when maximizing $\mathcal{L}$---or, equivalently, minimizing $L$---by the EM algorithm, the value of $L(\bth_t)$ depends on $\bth_t$ only through $\|\bth_t\|$. This means that we are interested in the dynamics of evolution not of $\bth_t$ itself, but of its norm $\|\bth_t\|$. In Appendix~\ref{app:m_radial}, we show that for the population EM operator $M(\bth)$ defined by \eqref{eq:pop_em_op}, we have
\begin{equation}
\|M(\bth)\|=\E_{Z\sim\mathcal{N}(0,1)}\left[t_p\left(\frac{\|\bth\|Z}{1-\frac{\|\bth\|^2}{d}}\right)Z\right].\label{eq:m_radial}
\end{equation}
Using the shorthand notation  $\theta_t:=\|\bth_t\|$, and in light of \eqref{eq:m_radial}, the sequence of norms $\theta_t$ satisfies $\theta_{t+1}=m(\theta_t)$, where
    \begin{equation}
        m(\theta):=\E_{Z\sim\mathcal{N}(0,1)}\left[t_p\left(\frac{\theta Z}{1-\theta^2/d}\right)Z\right].\label{eq:m_fun}
    \end{equation}

\paragraph{Properties of $\ell(\theta)$ and $m(\theta)$.}
The radiality of $L(\bth)$ and $\|M(\bth)\|$ allows us, instead of studying the population risk \eqref{eq:pop_risk} and population EM operator \eqref{eq:pop_em_op}, to study the univaritate functions $\ell(\theta)$ and $m(\theta)$. Some of their properties are given in the following lemma.

\begin{lemma} \label{lem:properties}
    Let $\ell(\theta)$ and $m(\theta)$ be the functions defined in \eqref{eq:ell} and \eqref{eq:m_fun}, respectively. Then for $\theta\in[0,\theta_0]$, where $\theta_0<\sqrt{d\cdot\frac{2+q-\sqrt{8q+q^2}}{2}}$ with $q:=1-\frac{(2p-1)^2}{2}$,
    \begin{enumerate}
        \item $m'(\theta)<1$,
        \item \label{lem:M_bounds}$m(\theta)\le\rho\cdot\theta$\quad for $\rho:=\frac{1+\theta_0^2/d}{(1-\theta_0^2/d)^2}\cdot\left(1-\frac{(2p-1)^2}{2}\right)\in(0,1)$,
        \item\label{lem:convex}  $\ell(\theta)$ is convex,
        \item \label{lem:pl} $
    [\ell'(\theta)]^2\ge(1-\rho)\cdot[\ell(\theta)-\ell(0)]$, where $\rho\in(0,1)$ is defined in Part~\ref{lem:M_bounds}.
    \end{enumerate}
\end{lemma}

Part~\ref{lem:pl} of Lemma~\ref{lem:properties} is key to establishing the exponential convergence of the mixture \eqref{eq:gmm} to $\mathcal{N}(\mathbf{0},\mathbf{I})$ in the KL distance. This is the so-called Polyak-Lojasiewicz inequality \citep{POLYAK1963864,lojasiewicz1963topological}, which is a sufficient condition for $\ell(\theta_t)-\ell(0)$ to converge exponentially to zero if we were to minimize $\ell(\theta)$ using the gradient descent method with a suitable choice of step size. Note that the EM algorithm is not identical to gradient descent, but is related to it, as discussed in \cite{DBLP:conf/nips/WangGNL15}, \cite{DBLP:journals/corr/BalakrishnanWY14}.

\paragraph{Convergence in KL distance.}
We are ready to prove the exponential decay of the KL divergence between the true distribution and the sequence of fitted mixtures.
\begin{theorem}\label{thm:kl_pop} For any $\boldsymbol{\theta}_0$ such that $\|\bth_0\|<\sqrt{d\cdot\frac{2+q-\sqrt{8q+q^2}}{2}}$ with $q:=1-\frac{(2p-1)^2}{2}$, we have
$$
\kl\left[\mathcal{N}(\mathbf{0},\mathbf{I})\parallel\mathcal{G}(\bth_T,\sigma^2_T)\right]\le\frac{1}{(1+c)^T}\cdot\kl\left[\mathcal{N}(\mathbf{0},\mathbf{I})\parallel\mathcal{G}(\bth_0,\sigma^2_0)\right],
$$  
where $c>0$.
\end{theorem}
\begin{proof} First of all notice that
$$
\kl\left[\mathcal{N}(\mathbf{0},\mathbf{I})\parallel\mathcal{G}(\bth_t,\sigma^2_t)\right]=L(\bth_t)-L(\mathbf{0}).
$$
Recall that $L(\bth)=\ell(\theta)$, where $\theta:=\|\bth\|$ (see \eqref{eq:ell_rad}). Thus, from Lemma~\ref{lem:properties} part~\ref{lem:convex}, we have
    \begin{align}
        L(\bth_{t})-L(\bth_{t+1})=\ell(\theta_t)-\ell(\theta_{t+1})&\ge\ell'(\theta_{t+1})(\theta_t-\theta_{t+1})\notag\\
        &=\ell'(\theta_{t+1})(\theta_t-m(\theta_t)).\label{eq:l_diff}
    \end{align}
    From \eqref{eq:lambda_prime}, taking into account that $\theta_{t+1}=m(\theta_t)\le\rho\theta_t<\theta_t$, and by convexity of $\ell$, we have
    \begin{equation}
        \theta_t-m(\theta_t)=\frac{(1-\theta_t^2/d)^2}{1+\theta_t^2/d}\cdot\ell'(\theta_t)\ge\underbrace{\frac{(1-\theta_0^2/d)^2}{1+\theta_0^2/d}}_{c_1(\theta_0,d)}\cdot\ell'(\theta_{t+1}).\label{eq:th_M_diff}
    \end{equation}
    Substituting \eqref{eq:th_M_diff} into \eqref{eq:l_diff} and taking into account Lemma~\ref{lem:properties} part~\ref{lem:pl}, we have
    $$
    \ell(\theta_t)-\ell(\theta_{t+1})\ge c_1(\theta_0,d)\left[\ell'(\theta_{t+1})\right]^2\ge c_2(\theta_0,d,p)\cdot[\ell(\theta_{t+1})-\ell(0)],
    $$
    where $c_2(\theta_0,d,p):=c_1(\theta_0,d)-\left(1-\frac{(2p-1)^2}{2}\right)$. Rearranging terms and subtracting $\ell(0)$ from both sides of the latter inequality we have
    \begin{align*}
    &\ell(\theta_t)-\ell(0)\ge\ell(\theta_{t+1})-\ell(0)+c_2(\theta_0,d,p)\cdot[\ell(\theta_{t+1})-\ell(0)]\\
    \Leftrightarrow\quad&\ell(\theta_{t+1})-\ell(0)\le\frac{1}{1+c_2(\theta_0,d,p)}[\ell(\theta_t)-\ell(0)]
    \end{align*}
    Applying this inequality recursively, we get
    $$
    \ell(\theta_T)-\ell(0)\le\frac{\ell(\theta_0)-\ell(0)}{(1+c_2(\theta_0,d,p))^T}
    $$
    which completes the proof.
\end{proof}
Note that to prove Theorem~\ref{thm:kl_pop} we did not use the convergence of the parameters themselves $(\bth_t,\sigma^2_t)$ to the optimal $(\mathbf{0},1)$. This happened because the PL inequality (Lemma~\ref{lem:properties} part~\ref{lem:pl}) provides a sufficient condition for the rapid convergence of the values of the function $\ell(\theta_t)$ to the optimal value $\ell(0)$. 
\paragraph{Numerical verification.} We verify the exponential convergence numerically. For this, we  consider $d=2$ and evaluate the cases of $p = 0.6, 0.8,$ and $0.9$ with the same starting value of $\bth_0=(0.20, 0.05)$. Larger values of $p$ lead to faster convergence as can be observed in Figure~\ref{fig:kl_conv}.
\begin{figure}
    \centering
    \includegraphics[width=170pt]{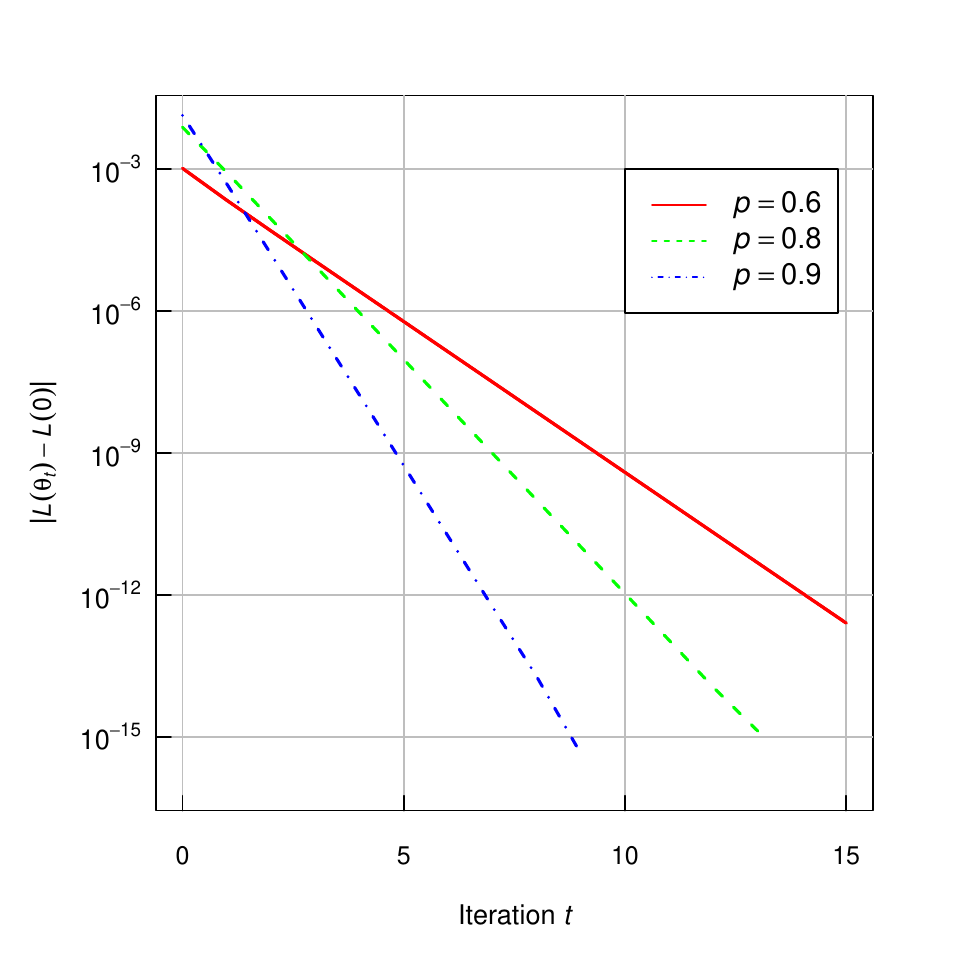}
    \caption{Plot of $\kl[\mathcal{N}(\mathbf{0},\mathbf{I}\parallel\mathcal{G}(\bth_t,\sigma^2_t)$ versus iteration number $t$.}
    \label{fig:kl_conv}
\end{figure}

\subsection{Finite-Sample Analysis} \label{sec:finite_sample}

In this section we show that the EM algorithm takes $O\left(\log(n)\right)$ steps to produce an estimate $(\hat\bth_t,\hat\sigma^2_t)$ such that $D_\text{KL}[\mathcal{N}(\mathbf{0},\mathbf{I})\parallel\mathcal{G}(\hat\bth_t,\hat\sigma^2_t)] = O(d/n)$ with high probability. More precisely, our key result is the following theorem.

\begin{theorem}\label{thm:main}
    Fix $\delta\in(0,1)$, $\alpha\in(0,1/2)$, and let $n\succsim \log^{\frac1{2\alpha}}\left(1/\delta\right)$. Then with any starting point $\bth_0$ such that $\|\boldsymbol{\theta}_0\|<\min\left[\sqrt{d\cdot\frac{2+q-\sqrt{8q+q^2}}{2}},\frac{1}{1+\sqrt{1+1/d}}\right]$ where $q:=1-\frac{(2p-1)^2}{2}$, the EM algorithm produces a sequence of parameter estimates $(\hat\bth_t,\hat\sigma^2_t)$ such that
    \begin{equation}
        D_\text{KL}\left[\mathcal{N}(\mathbf{0},\mathbf{I})\parallel\mathcal{G}(\hat\bth_T,\hat\sigma^2_T)\right]\le c_1\|\bth_0\|^2\frac{\log(1/\delta)}{n}\label{eq:kl_bound}
    \end{equation}
    for $T\ge c_2\log\frac{n}{\log(1/\delta)}$ with probability at least $1-\delta$.
\end{theorem}

    Figure~\ref{fig:kl_sample} reports on empirical verification of the bound \eqref{eq:kl_bound}. As we can see, the KL divergence between $\mathcal{N}(\mathbf{0},\mathbf{I})$ and $\mathcal{G}(\hat\bth_T,\hat\sigma^2_T)$ indeed drops at a rate of $O(1/n)$.
\begin{figure}
    \centering
    \includegraphics[width=170pt]{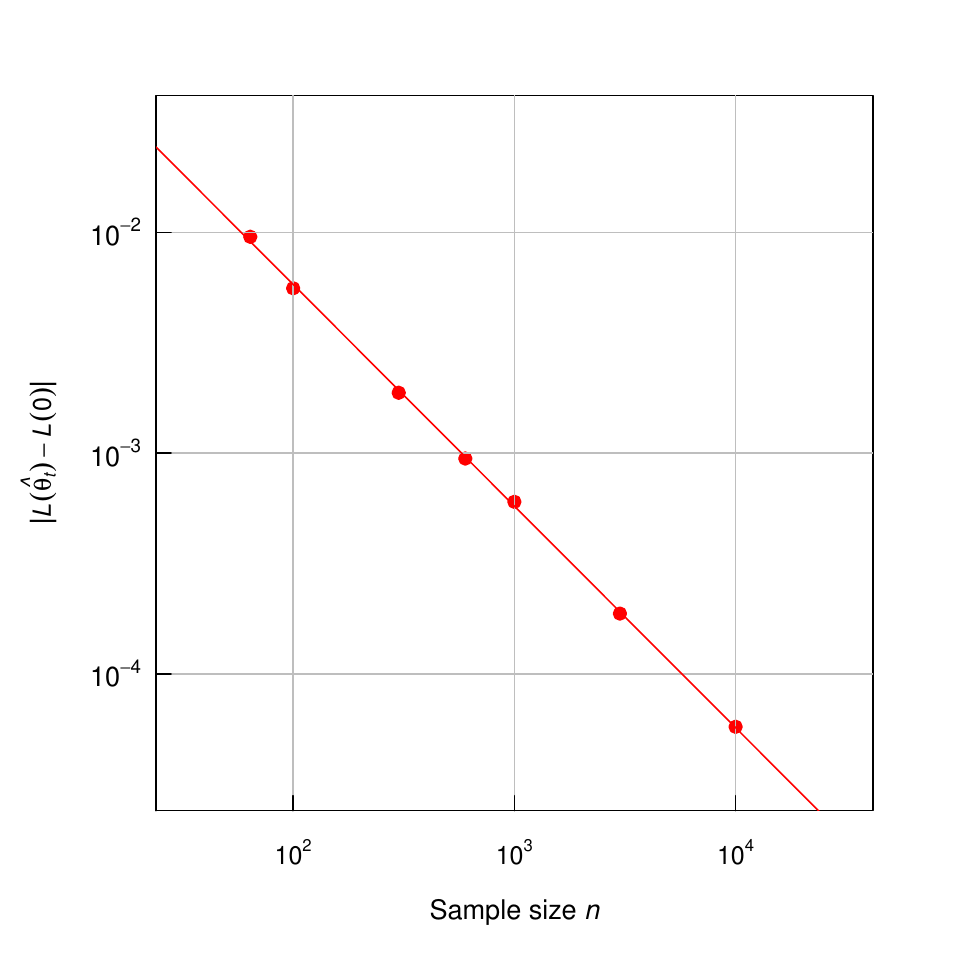}
    \caption{Plot of the KL divergence $D_\text{KL}\left[\mathcal{N}(\mathbf{0},\mathbf{I})\parallel\mathcal{G}(\hat\bth_T,\hat\sigma^2_T)\right]$ versus the sample size $n$. We consider the case $d=2$, and use the starting value $\bth_0=(0.20, 0.05)$. For each $n$ we generate a sample of size $n$ from $\mathcal{N}(\mathbf{0},\mathbf{I})$ and use the EM algorithm to fit the balanced two-component mixture $\mathcal{G}(\bth,\sigma^2)$ to it. We repeat this process 10 times and report the average value of the KL divergence across those 10 runs. The slope of the fitted line is $-1.004$.}
    \label{fig:kl_sample}
\end{figure}

Denote the sample-based averaged log-likelihood in \eqref{eq:mle} as
\begin{equation}
    \mathcal{L}_n(\bth,\sigma^2):=\frac1n\sum_{i=1}^n\log f(\mathbf{Z}_i;\bth,\sigma^2).\label{eq:sample_log_l}
\end{equation}
When it is maximized by the EM algorithm, the parameter updates can also be written explicitly (as in the case of the population EM, see Appendix~\ref{app:sample_em_upd}):
    \begin{align}
        \hat{\bth}_{t+1}&=\frac1n\sum_{i=1}^n t_p\left(\frac{\hat\bth_t^\top\mathbf{Z}_i}{\sfrac{\sum_{i=1}^n\|\mathbf{Z}_i\|^2}{nd}-\sfrac{\|\hat\bth_t\|^2}{d}}\right)\mathbf{Z}_i,\label{eq:theta_hat_upd}\\
        \hat{\sigma}^2_{t+1}&=\frac{\sum_{i=1}^n\|\mathbf{Z}_i\|^2}{nd}-\frac{\|\hat\bth_{t+1}\|^2}{d}.\label{eq:sigma_sq_hat_upd}
    \end{align}
As we can see, the evolution of the sample-based estimate $\hat\bth_t$ is governed by the operator $M_n:\mathbb{R}^d\to\mathbb{R}^d$ defined as
$$
    M_n(\bth):=\frac1n\sum_{i=1}^n t_p\left(\frac{\bth^\top\mathbf{Z}_i}{\sfrac{\sum_{i=1}^n\|\mathbf{Z}_i^2\|}{nd}-\sfrac{\|\bth\|^2}d}\right)\mathbf{Z}_i
$$
Let $Z_i:=\sfrac{\bth^\top\mathbf{Z}_i}{\|\bth\|}$, then $Z_1,\ldots,Z_n\,\,{\stackrel{\text{iid}}{\sim}}\,\,\mathcal{N}(0,1)$. By analogy with the population case, we can move from the operator $M_n$ to the function of one variable $m_n(\theta)$ given by
$$   m_n(\theta):=\frac1n\sum_{i=1}^n t_p\left(\frac{\theta Z_i}{\sfrac{U_{nd}}{nd}-\sfrac{\theta^2}{d}}\right)Z_i,
$$
where $U_{nd}:=\sum_{i=1}^n\|\mathbf{Z}_i\|^2$ is a chi-square random variable with $nd$ degrees of freedom. 

Under the assumptions of Theorem~\ref{thm:main}, the following perturbation bound relates the sample-based EM operator to its population-level counterpart:
\begin{equation}
    \Pr\left[\sup_{\theta\in[0,r]}|{m}_n(\theta)-{m}(\theta)|\le cr\sqrt{\frac{\log(1/\delta)}{n}}\right]\ge1-\delta.\label{eq:perturb}
\end{equation}
The proof idea of the bound~\eqref{eq:perturb} is due to \cite{dwivedi2020singularity} and  \cite{DBLP:conf/aistats/DwivediHKWJ020}, and utilizes conventional methods for establishing Rademacher complexity bounds. Initially, we employ symmetrization using Rademacher variables and apply the Ledoux-Talagrand contraction inequality. Subsequently, we use findings related to the chi-square distribution and carry out computations related to Chernoff bounds to achieve the intended outcome. Full proof is given in Appendix~\ref{app:perturb}.

Thanks to the strict contractivity of the population EM operator (Lemma~\ref{lem:properties} part~\ref{lem:M_bounds}) and the perturbation bound \eqref{eq:perturb}, we can show that the sequence of EM iterates $\hat{\boldsymbol{\theta}}_{t+1}=M_n(\hat{\boldsymbol{\theta}}_t)$ satisfies
\begin{equation}
\|\hat{\boldsymbol{\theta}}_T\|\le \frac{c\|\boldsymbol{\theta}_0\|}{1-\rho}\sqrt{\frac{\log(1/\delta)}{n}}\label{eq:theta_conv}
\end{equation}
for $T\ge\log_{1/\rho}\left((1-\rho)\sqrt{\frac{n}{\log(1/\delta)}}\right)$ with probability at least $1-\delta$. The proof is given in Appendix~\ref{app:theta_convergence}. This is a new result. In the work of \cite{dwivedi2020singularity}, a similar result was obtained for the case when the variance $\sigma^2$ is known and is not learned from the data. The work by \cite{DBLP:conf/aistats/DwivediHKWJ020} considers the case where $\sigma^2$ is learned (along with $\bth$), but a mixture is balanced (i.e. when $p=1/2$ in \eqref{eq:gmm}).

We need the following lower bound for $m(\theta)$ to upperbound the derivative $\ell'(\theta)$ later: for $\theta\in\left[0,\left(\sqrt2+\frac1{\sqrt2d}\right)^{-1}\right]$, we have (see Appendix~\ref{app:lowerbound})
    \begin{equation}
    m(\theta)\ge4p(1-p)\left(1-\frac{4\theta^2}{(1-\theta^2/d)^2}\right)\theta.\label{eq:LB}
    \end{equation}
Finally, we are ready to prove our key result on convergence in KL distance for the finite-sample case.

\begin{proof}[Proof of Theorem~\ref{thm:main}] 

    First of all, notice that
    \begin{equation}
        \kl\left[\mathcal{N}(\mathbf{0},\mathbf{I})\parallel \mathcal{G}(\hat\bth_t,\hat\sigma^2_t)\right]=\E_{\mathbf{Z}\sim\mathcal{N}(\mathbf{0},\mathbf{I})}\left[\log\frac{\phi(\mathbf{Z})}{f(\mathbf{Z};\hat\bth_t,\hat\sigma^2_t)}\right]=\mathcal{L}(\mathbf{0},\mathbf{I})-\mathcal{L}(\hat\bth_t,\hat\sigma^2_t),\label{eq:kl_main1}
    \end{equation}
    where $\mathcal{L}(\bth,\sigma^2)$ is the population-level expected log-likelihood \eqref{eq:e_log_l}. Note that the first term in the update rule for $\sigma^2$ in \eqref{eq:sigma_sq_hat_upd} is not exactly one (but concentrates around one since $\sum_{i=1}^n\|\mathbf{Z}_i\|^2\sim\chi^2_{nd}$), and thus  the point $(\hat\bth_t,\hat\sigma^2_t)$ does not lie on the hypersurface $\mathcal{S}$ defined by \eqref{eq:hypersurf}. This does not allow us to move from $\mathcal{L}(\hat\bth_t,\hat\sigma^2_t)$ directly to $\mathcal{L}\left(\hat\bth_t,1-\|\hat\bth_t\|^2/d\right)$. However, let us consider their difference
    \begin{equation}
        \mathcal{L}\left(\hat\bth_t,1-\|\hat\bth_t\|^2/d\right)-\mathcal{L}\left(\hat\bth_t,\hat\sigma^2_t\right) =\frac{\partial\mathcal{L}}{\partial\sigma^2}\left(\hat\bth_t,\xi\right)\left(\frac{U}{nd}-1\right),\label{eq:kl_main2}
    \end{equation}
    where $U:=\sum_{i=1}^n\|\mathbf{Z}_i\|^2\sim\chi^2_{nd}$, and $\xi$ is a point between $\frac{U}{nd}-\frac{\hat\bth_t^2}{d}$ and $1-\frac{\hat\bth_t^2}{d}$. Define the event $\mathcal{A}_\epsilon:=\left\{\left|\frac{U}{nd}-1\right|\le\epsilon\right\}$, and notice that $\Pr[\mathcal{A}_\epsilon]\ge1-2e^{-\epsilon^2nd/8}$ by standard chi-square tail bounds. Conditional on $\mathcal{A}_\epsilon$, the point $\xi$ is in the interval $1-\|\hat\bth_t\|^2/d\pm\epsilon$. Thus, by direct calculation and elementary inequalities,
    \begin{align}
        &\left|\frac{\partial\mathcal{L}}{\partial\sigma^2}\left(\hat\bth_t,\xi\right)\right|=\left|\frac{d}{2\xi}-\frac{d+\hat\bth^2_t}{2\xi^2}+\frac{1}{\xi^2}\E_{\mathbf{Z}\sim\mathcal{N}(\mathbf{0},\mathbf{I})}\left[\tanh\left(\frac{\hat\bth_t^\top \mathbf{Z}}{\xi}\right)\hat\bth_t^\top\mathbf{Z}\right]\right|\notag\\
        &\le\left|\frac{d}{2\left(1-\frac{\hat\theta_t^2}d\right)}-\frac{d+\hat\theta^2_t}{2\left(1-\frac{\hat\theta_t^2}{d}\right)^2}+\frac{1}{\left(1-\frac{\hat\theta_t^2}{d}\right)^2}\E\left[\tanh\left(\frac{\hat\theta_t{Z}}{1-\frac{\hat\theta^2}{d}}\right)\hat\theta_t{Z}\right]\right|+c\epsilon\notag\\
        &=\frac{\hat\theta_t\left[\hat\theta_t-m(\hat\theta_t)\right]}{(1-\hat\theta^2_t/d)^2}+c\epsilon\le\frac{\left(1-4p(1-p)\left\{1-\frac{4\theta^2_0}{(1-\theta^2_0/d)^2}\right\}\right)\cdot\hat\theta_t^2}{(1-\theta_0^2/d)^2}\precsim\hat\theta^2_t+\epsilon,\label{eq:kl_main3}
    \end{align}
    where $\hat\theta_t=\|\hat{\bth}_t\|$, and  we used the lower bound for $m(\theta)$ from \eqref{eq:LB}.
    Solving $2e^{-\epsilon^2(nd)/8}=\delta$ for $\epsilon$, and then combining \eqref{eq:kl_main2} and \eqref{eq:kl_main3}, we have
    \begin{equation}
        \left|\mathcal{L}\left(\hat\bth_t,1-\|\hat\bth_t\|^2/d\right)-\mathcal{L}\left(\hat\bth_t,\hat\sigma^2_t\right)\right|\precsim\hat\theta^2_t\sqrt{\frac{\log\frac1\delta}{nd}}+\frac{\log\frac1\delta}{nd}\label{eq:kl_main4}
    \end{equation}
    with probability at least $1-\delta$. 
    Thus, instead of $\mathcal{L}(\hat\bth_t,\hat\sigma^2_t)$ we can focus on $\mathcal{L}(\hat\bth_t,1-\|\hat\bth_t\|^2/{d})=:-L(\hat\bth_t)$. 

By radiality of $L(\bth)$ (see \eqref{eq:ell_rad}) and convexity of $\ell(\theta)$ (Lemma~\ref{lem:properties} part~\ref{lem:convex}), we have
\begin{equation}
    L(\hat{\boldsymbol{\theta}}_t)-L(\mathbf{0})=\ell(\hat\theta_t)-\ell(0)\le\ell'(\hat\theta_t)\cdot\hat\theta_t,\label{eq:kl_main5}
\end{equation}
By direct calculation (see \eqref{eq:lambda_prime}) and using \eqref{eq:LB}, the derivative $\ell'(\hat\theta_T)$ can be bounded as
\begin{align}
    \ell'(\hat\theta_t)&=\frac{1+\hat\theta_t^2/d}{(1-\hat\theta^2_t/d)^2}\cdot(\hat\theta_t-m(\hat\theta_t))\notag\\
    &\le\frac{1+\theta_0^2/d}{(1-\theta_0^2/d)^2}\cdot\left(1-4p(1-p)\left\{1-\frac{4\theta^2_0}{(1-\theta^2_0/d)^2}\right\}\right)\cdot\hat\theta_t.\label{eq:kl_main6}
\end{align}
From \eqref{eq:kl_main5}, \eqref{eq:kl_main6}, we obtain
$$
L(\hat\bth_t)-L(\mathbf{0})\precsim\hat\theta_t^2\label{eq:kl_main7}
$$
Let $n\succsim \log^{\frac1{2\alpha}}\left(1/\delta\right)$. Then by \eqref{eq:theta_conv}, $\theta_T\precsim\|\bth_0\|\sqrt{\frac{\log(1/\delta)}{n}}$ for $T\succsim\log{\frac{n}{\log(1/\delta)}}$ with probability at least $1-\delta$. Hence, for such $T$ and $n$, combining \eqref{eq:kl_main1}, \eqref{eq:kl_main4}, and \eqref{eq:kl_main7} we have
$$
\kl\left[\mathcal{N}(\mathbf{0},\mathbf{I})\parallel \mathcal{G}(\hat\bth_T,\hat\sigma^2_T)\right]\precsim\|\bth_0\|^2\frac{\log(1/\delta)}{n}
$$
with probability at least $1-\delta$, which completes the proof. 
\end{proof}

\section{Proof of Main Results}\label{sec:main_proof}
We start with an auxiliary result on binary classification.  Let $(\mathbf{X},Y)$ be a pair of random variables taking their respective values from $\mathbb{R}^d$ and $\{-1,+1\}$. Define a function
$$
\eta(\mathbf{x}):=\Pr[Y=+1\mid\mathbf{X}=\mathbf{x}],
$$
which is sometimes called the \emph{a posteriori probability}. Any function $h:\mathbb{R}^d\to\{-1,+1\}$ defines a \emph{classifier}. The misclassification error rate of $h$ is 
$$
\Err[h]:=\Pr[h(\mathbf{X})\ne Y],
$$
which is minimized by the Bayes classifier \cite[Theorem 2.1]{DBLP:books/sp/DevroyeGL96}
$$
h^\ast(\mathbf{x})=\begin{cases}
    +1\quad&\text{if }\eta(\mathbf{x})>1-\eta(\mathbf{x})\\
    -1&\text{otherwise}.
\end{cases}
$$
The misclassification error rate of the Bayes classifier $\Err[h^\ast]$ is referred to as Bayes risk. In the case of balanced classes, i.e. $\Pr[Y=-1]=\Pr[Y=+1]=1/2$, if the class-conditional densities $f_-(\mathbf{x})$ and $f_+(\mathbf{x})$ are not known and are approximated by $\tilde{f}_-(\mathbf{x})$ and $\tilde{f}_+(\mathbf{x})$, respectively, it is natural to use the plug-in classifier
$$
h(\mathbf{x})=\begin{cases}
    +1\quad&\text{if }\tilde{f}_+(\mathbf{x})>\tilde{f}_-(\mathbf{x})\\
    -1&\text{otherwise}.
\end{cases}
$$
to approximate the Bayes classifier. The next lemma states that if $\tilde{f}_+(\mathbf{x})$ is close to $f_+(\mathbf{x})$ and $\tilde{f}_-(\mathbf{x})$ is close to $f_-(\mathbf{x})$ in $L_1$-sense, then misclassification error of $h(\mathbf{x})$ is near the Bayes risk.

\begin{lemma}\label{lem:excess}
Assume that the true class-conditional densities $f_-(\mathbf{x})$, $f_+(\mathbf{x})$ are approximated by the densities $\tilde{f}_-(\mathbf{x})$, $\tilde{f}_+(\mathbf{x})$, respectively. Let $\Pr[Y=+1]=\Pr[Y=-1]=\frac12$. Then for the plug-in classifier $h$ defined above
$$
\Err[h]-\Err[h^\ast]\le\frac12\int_{\mathbb{R}^d}\left|f_-(\mathbf{x})-\tilde{f}_-(\mathbf{x})\right|d\mathbf{x}+\frac12\int_{\mathbb{R}^d}\left|f_+(\mathbf{x})-\tilde{f}_+(\mathbf{x})\right|d\mathbf{x}.
$$
\end{lemma}

In the context of over-specified MDA that we consider in this work, the true class-conditional distributions are $\mathcal{N}(\pm\boldsymbol{\mu},\mathbf{I})$ and they are approximated by unbalanced Gaussian mixtures
$\mathcal{G}(\pm\boldsymbol{\mu},\bth,\sigma^2)$ defined in \eqref{eq:gmm_class}. The sequences of population-level and sample-based MDA classifiers fit by the EM algorithm are specified in \eqref{eq:mda_clf_pop} and \eqref{eq:mda_clf}, respectively. 
Notice that the Bayes classifier in this case is the LDA classifier 
$$
h^\ast(\mathbf{x})=\begin{cases}
    1\quad&\text{if }\phi(\mathbf{x}-\boldsymbol{\mu})>\phi(\mathbf{x}+\boldsymbol{\mu})\\
    0&\text{otherwise},
\end{cases}
$$
and its error rate (i.e., Bayes risk) is 
\begin{equation}
\Err[h^\ast]= \Phi(-\|\boldsymbol{\mu}\|),\label{eq:tv0}
\end{equation}
see Section 1.2.1 in \cite{Wainwright_2019}.

\vspace{10pt}

\subsection{Proof of Theorem~\ref{thm:mda_pop}}
\begin{proof}
    Since $\sigma^2_t$ is fully determined by $\bth_t$ \eqref{eq:sigma_sq_hat_upd}, we will write $\mathcal{G}(\boldsymbol{\mu},\bth_t)$ instead of $\mathcal{G}(\boldsymbol{\mu},\bth_t,\sigma^2_t)$. At the population level, we can assume that $\boldsymbol{\mu}$ is known since $\boldsymbol{\mu}=\E[\mathbf{X}Y]$. By Lemma~\ref{lem:excess}, we have
    \begin{equation}
    \Err[h_t]-\Err[h^\ast]\le\tv[\mathcal{N}(-\boldsymbol{\mu},\mathbf{I}),\mathcal{G}(-{\boldsymbol{\mu}},{\bth}_t)]+\tv[\mathcal{N}(\boldsymbol{\mu},\mathbf{I}),\mathcal{G}({\boldsymbol{\mu}},{\bth}_t)].\label{eq:tv1_pop}
    \end{equation}
    We will bound the second term, and the first term can be bounded analogously. 
    By Pinsker's inequality, and from Theorem~\ref{thm:kl_pop}, we have
    \begin{align}
        \tv\left[\mathcal{N}(\boldsymbol\mu,\mathbf{I}),\mathcal{G}({\boldsymbol\mu},\bth_T)\right]&\le\sqrt{\frac12\kl\left[\mathcal{N}(\boldsymbol{\mu},\mathbf{I})\parallel\mathcal{G}(\boldsymbol\mu,\bth_T)\right]}\notag\\
        &=\sqrt{\frac12\E_{\mathbf{X}\sim\mathcal{N}(\boldsymbol\mu,\mathbf{I})}\left[\log\frac{\phi\left(\mathbf{X}-\boldsymbol{\mu}\right)}{f(\mathbf{X}-\boldsymbol\mu;\bth_T,1-\sfrac{\|\bth_T\|^2}{d})}\right]}\notag\\
        &=\sqrt{\frac12\E_{\mathbf{Z}\sim\mathcal{N}(\mathbf{0},\mathbf{I})}\left[\log\frac{\phi\left(\mathbf{Z}\right)}{f(\mathbf{Z};\bth_T,1-\sfrac{\|\bth_T\|^2}{d})}\right]}\notag\\
        &=\sqrt{\frac12\kl\left[\mathcal{N}(\mathbf{0},\mathbf{I})\parallel\mathcal{G}(\mathbf{0},\bth_T)\right]}\le\epsilon\label{eq:tv3_pop}
    \end{align}
    for $T\succsim\log\frac1\epsilon$. Combining \eqref{eq:tv0}, \eqref{eq:tv1_pop}, and  \eqref{eq:tv3_pop}, the statement of the theorem follows.
\end{proof}

\subsection{Proof of Theorem~\ref{thm:mda}}
\begin{proof}
    For the sake of brevity, we denote $\boldsymbol{\nu}:=(\bth,\sigma^2)$, and we will write $\mathcal{G}(\boldsymbol{\mu},\boldsymbol{\nu})$ instead of $\mathcal{G}(\boldsymbol{\mu},\bth,\sigma^2)$. By Lemma~\ref{lem:excess}, we have
    \begin{equation}
    \Err[h_t]-\Err[h^\ast]\le\tv[\mathcal{N}(-\boldsymbol{\mu},\mathbf{I}),\mathcal{G}(-\hat{\boldsymbol{\mu}},\hat{\boldsymbol{\nu}}_t)]+\tv[\mathcal{N}(\boldsymbol{\mu},\mathbf{I}),\mathcal{G}(\hat{\boldsymbol{\mu}},\hat{\boldsymbol{\nu}}_t)].\label{eq:tv1}
    \end{equation}
    We will bound the second term, and the first term can be bounded analogously. By triangle inequality,
    \begin{equation}
    \tv\left[\mathcal{N}(\boldsymbol\mu,\mathbf{I}),\mathcal{G}(\hat{\boldsymbol\mu},\hat{\boldsymbol{\nu}}_T)\right]\le    \tv\left[\mathcal{N}(\boldsymbol\mu,\mathbf{I}),\mathcal{G}({\boldsymbol\mu},\hat{\boldsymbol{\nu}}_T)\right]+    \tv\left[\mathcal{G}(\boldsymbol\mu,\hat{\boldsymbol{\nu}}_T),\mathcal{G}(\hat{\boldsymbol\mu},\hat{\boldsymbol{\nu}}_T)\right]\label{eq:tv2}    
    \end{equation}
    Next, by Pinsker's inequality, and from Theorem~\ref{thm:main}, assuming $n\succsim\log^{\frac1{2\alpha}}(1/\delta)$ we have
    \begin{align}
        \tv\left[\mathcal{N}(\boldsymbol\mu,\mathbf{I}),\mathcal{G}({\boldsymbol\mu},\hat{\boldsymbol{\nu}}_T)\right]&\le\sqrt{\frac12\kl\left[\mathcal{N}(\boldsymbol{\mu},\mathbf{I})\parallel\mathcal{G}(\boldsymbol\mu,\hat{\boldsymbol{\nu}}_T)\right]}\notag\\&=\sqrt{\frac12\E_{\mathbf{X}\sim\mathcal{N}(\boldsymbol\mu,\mathbf{I})}\left[\log\frac{\phi\left(\mathbf{X}-\boldsymbol{\mu}\right)}{f(\mathbf{X}-\boldsymbol\mu;\hat{\boldsymbol{\nu}}_T)}\right]}\notag\\&=\sqrt{\frac12\E_{\mathbf{Z}\sim\mathcal{N}(\mathbf{0},\mathbf{I})}\left[\log\frac{\phi\left(\mathbf{Z}\right)}{f(\mathbf{Z};\hat{\boldsymbol{\nu}}_T)}\right]}\notag\\&=\sqrt{\frac12\kl\left[\mathcal{N}(\mathbf{0},\mathbf{I})\parallel\mathcal{G}(\mathbf{0},\hat{\boldsymbol{\nu}}_T)\right]}\precsim\|\bth_0\|\sqrt{\frac{\log(1/\delta)}{n}}\label{eq:tv3}
    \end{align}
    for $T\succsim\log\frac{n}{\log(1/\delta)}$ with probability at least $1-\delta$. By properties of a metric,
    \begin{align}
        &\tv\left[\mathcal{G}(\boldsymbol\mu,\hat{\boldsymbol{\nu}}_T),\mathcal{G}(\hat{\boldsymbol\mu},\hat{\boldsymbol{\nu}}_T)\right]\notag\\
        &=\tv\left[\frac12\mathcal{N}(\boldsymbol\mu-\hat\bth_T,\hat\sigma^2_T\mathbf{I})+\frac12\mathcal{N}(\boldsymbol\mu+\hat\bth_T,\hat\sigma^2_T\mathbf{I}),\frac12\mathcal{N}(\hat{\boldsymbol\mu}-\hat\bth_T,\hat\sigma^2_T\mathbf{I})+\frac12\mathcal{N}(\hat{\boldsymbol\mu}+\hat\bth_T,\hat\sigma^2_T\mathbf{I})\right]\notag\\
        &\le\frac12\tv\left[\mathcal{N}(\boldsymbol\mu-\hat\bth_T,\hat\sigma^2_T\mathbf{I}),\mathcal{N}(\hat{\boldsymbol\mu}-\hat\bth_T,\hat\sigma^2_T\mathbf{I})\right]+\frac12\tv\left[\mathcal{N}(\boldsymbol\mu+\hat\bth_T,\hat\sigma^2_T\mathbf{I}),\mathcal{N}(\hat{\boldsymbol\mu}+\hat\bth_T,\hat\sigma^2_T\mathbf{I})\right]\notag
        \end{align}
        \begin{align}
        &\le\frac12\sqrt{\frac12\kl\left[\mathcal{N}(\boldsymbol\mu-\hat\bth_T,\hat\sigma^2_T\mathbf{I})\parallel\mathcal{N}(\hat{\boldsymbol\mu}-\hat\bth_T,\hat\sigma^2_T\mathbf{I})\right]}\notag\\
        &\qquad\qquad+\frac12\sqrt{\frac12\kl\left[\mathcal{N}(\boldsymbol\mu+\hat\bth_T,\hat\sigma^2_T\mathbf{I})\parallel\mathcal{N}(\hat{\boldsymbol\mu}+\hat\bth_T,\hat\sigma^2_T\mathbf{I})\right]}\notag\\
        &=\frac12\sqrt{\frac14\cdot\frac{\|\boldsymbol{\mu}-\hat{\boldsymbol\mu}\|^2}{\hat\sigma^2_T}}+\frac12\sqrt{\frac14\cdot\frac{\|\boldsymbol{\mu}-\hat{\boldsymbol\mu}\|^2}{\hat\sigma^2_T}}\le\frac12\cdot\frac{\|\boldsymbol{\mu}-\hat{\boldsymbol\mu}\|}{\sqrt{\hat\sigma_0^2}}\precsim\sqrt{\frac{d}n\log\frac1\delta}\label{eq:tv4}
    \end{align}
    with probability at least $1-\delta$  which follows from the standard tail bounds for chi-square random variable $n\|\hat{\boldsymbol\mu}-\boldsymbol\mu\|^2$ with $d$ degrees of freedom. Combining \eqref{eq:tv0}, \eqref{eq:tv1}, \eqref{eq:tv2}, \eqref{eq:tv3}, and \eqref{eq:tv4}, the statement of the theorem follows.
\end{proof}

\section{Experiments}\label{sec:experiments}

To demonstrate the practical relevance of our theoretical findings, we conducted experiments on two remote sensing datasets: Salinas-A\footnote{\url{https://www.ehu.eus/ccwintco/index.php/Hyperspectral_Remote_Sensing_Scenes}} and EuroSAT  \citep{helber2019eurosat}.

\subsection{Salinas-A: pixel-wise classification}

We conducted a pixel-wise classification experiment on the Salinas-A hyperspectral scene, a $86\times 83$-pixel subimage with 6 vegetation classes captured by a 224-band sensor. 
\begin{figure}[t]
    \centering
    \includegraphics[width=0.8\textwidth]{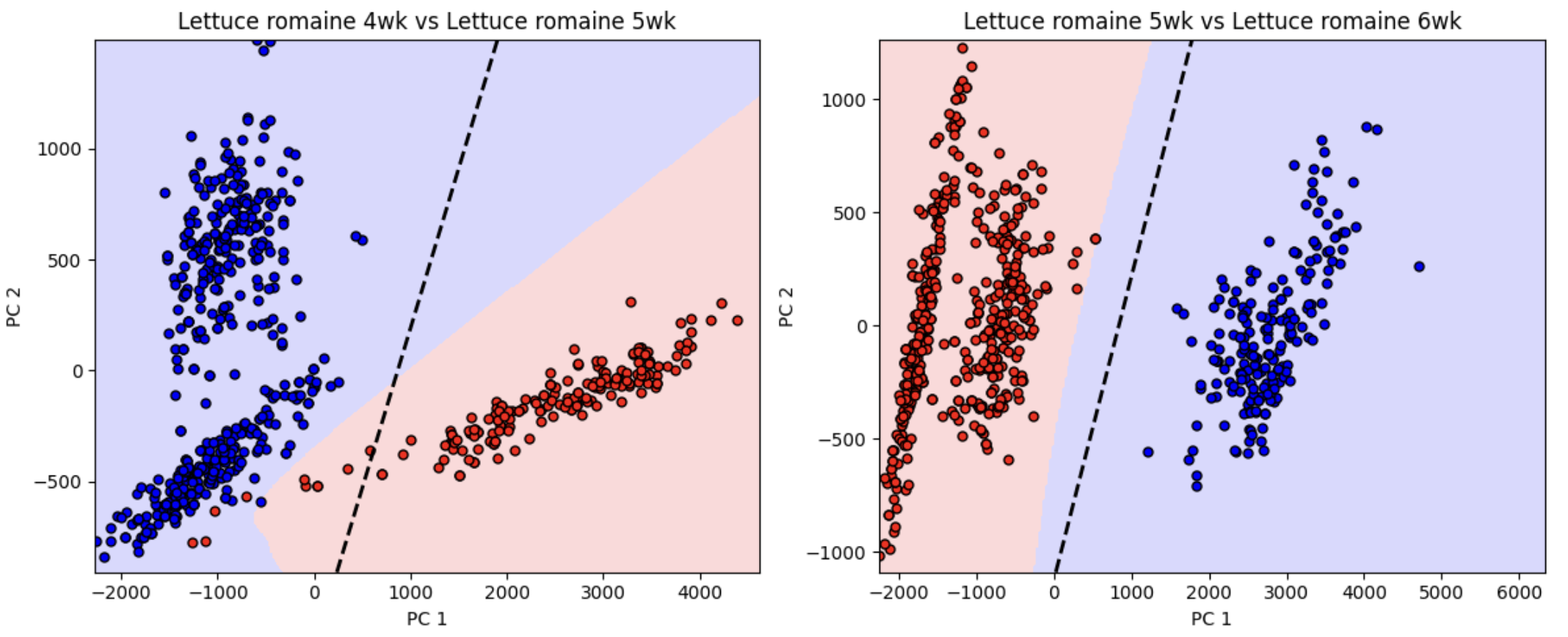}
    \caption{Decision boundaries for LDA vs. MDA on the two most confused class pairs (Lettuce romaine 4wk vs 5wk; 5wk vs 6wk), projected onto a 2D PCA subspace of the spectral features. Red and blue points denote test pixels from each class. The black dashed line shows LDA’s linear boundary, while the shaded regions indicate MDA’s two-component decision regions.}
    \label{fig:salinas}
\end{figure}\begin{figure}[t]
    \centering
    \includegraphics[width=0.8\textwidth]{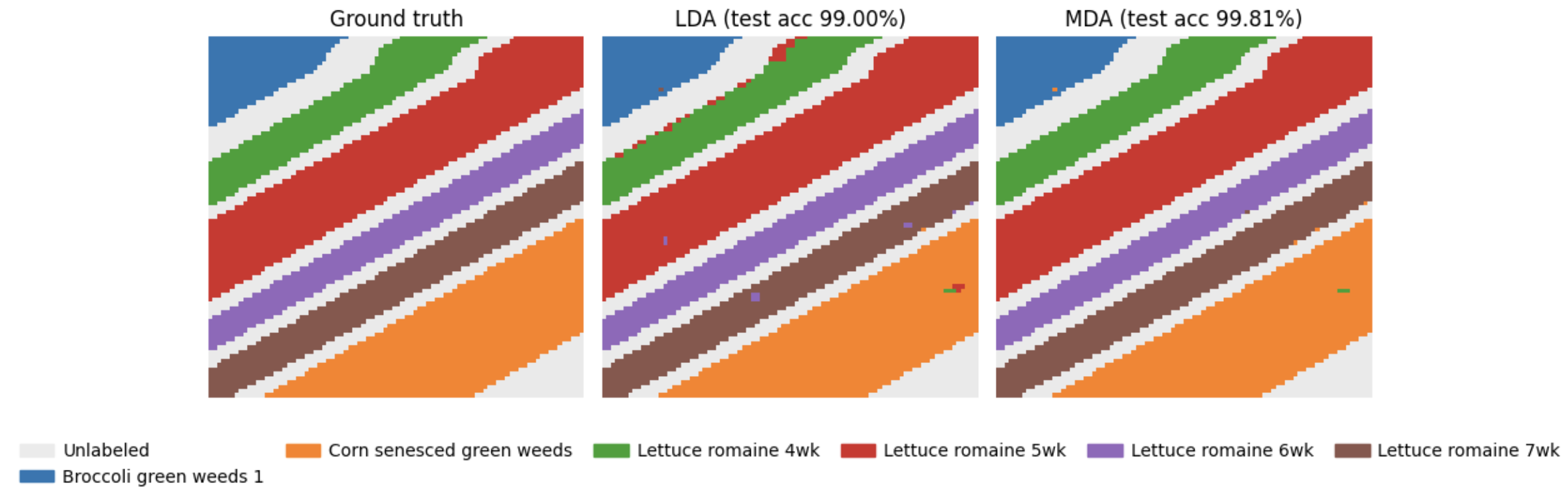}
    \caption{Ground-truth and predicted class label maps for the Salinas-A scene. Left: ground-truth map of the 6 classes; Center: LDA classification; Right: MDA classification. The MDA map corrects many of LDA’s errors, yielding a closer match to the true class distribution (especially at boundaries between confused classes).}
    \label{fig:salinas-maps}
\end{figure}

All spectral features (204 effective bands after preprocessing) were reduced to 30 dimensions via PCA. We trained LDA and MDA (two Gaussian components per class, tied covariance) on 70\% of the labeled pixels and tested on the remaining 30\%. The LDA achieved a test accuracy of 0.990, while MDA reached 0.998. This 0.8 percentage-point gain is statistically significant (p = 0.0002, McNemar’s test), with MDA correctly classifying 13 samples that LDA misclassified (and never erring where LDA was correct). Figure~\ref{fig:salinas} illustrates the classifiers’ decision boundaries on a 2D PCA projection for the two class pairs most frequently confused by LDA (Lettuce romaine 4wk vs 5wk; 5wk vs 6wk). MDA’s flexible two-component model provides a more nuanced boundary that resolves many of LDA’s errors for this pair, though their overall decision regions remain similar given the high separability. Figure~\ref{fig:salinas} shows clearly non–LDA-like geometry (elongated, non-elliptical clusters); thus this setting lies outside Theorem~\ref{thm:mda}’s assumptions. We include it to chart the boundary of the theory: when assumptions fail, overspecified MDA is beneficial (here improving over LDA), while under LDA-like data the two coincide.

Figure~\ref{fig:salinas-maps} compares the spatial classification maps of the entire scene: the MDA prediction aligns more closely with the ground truth, correcting the few misclassified regions present in the LDA map (mainly along field boundaries between spectrally similar crops).

\subsection{EuroSAT Satellite Image Classification}

\begin{figure}[htbp]
\begin{minipage}{\textwidth}
\centering
\includegraphics[width=\textwidth]{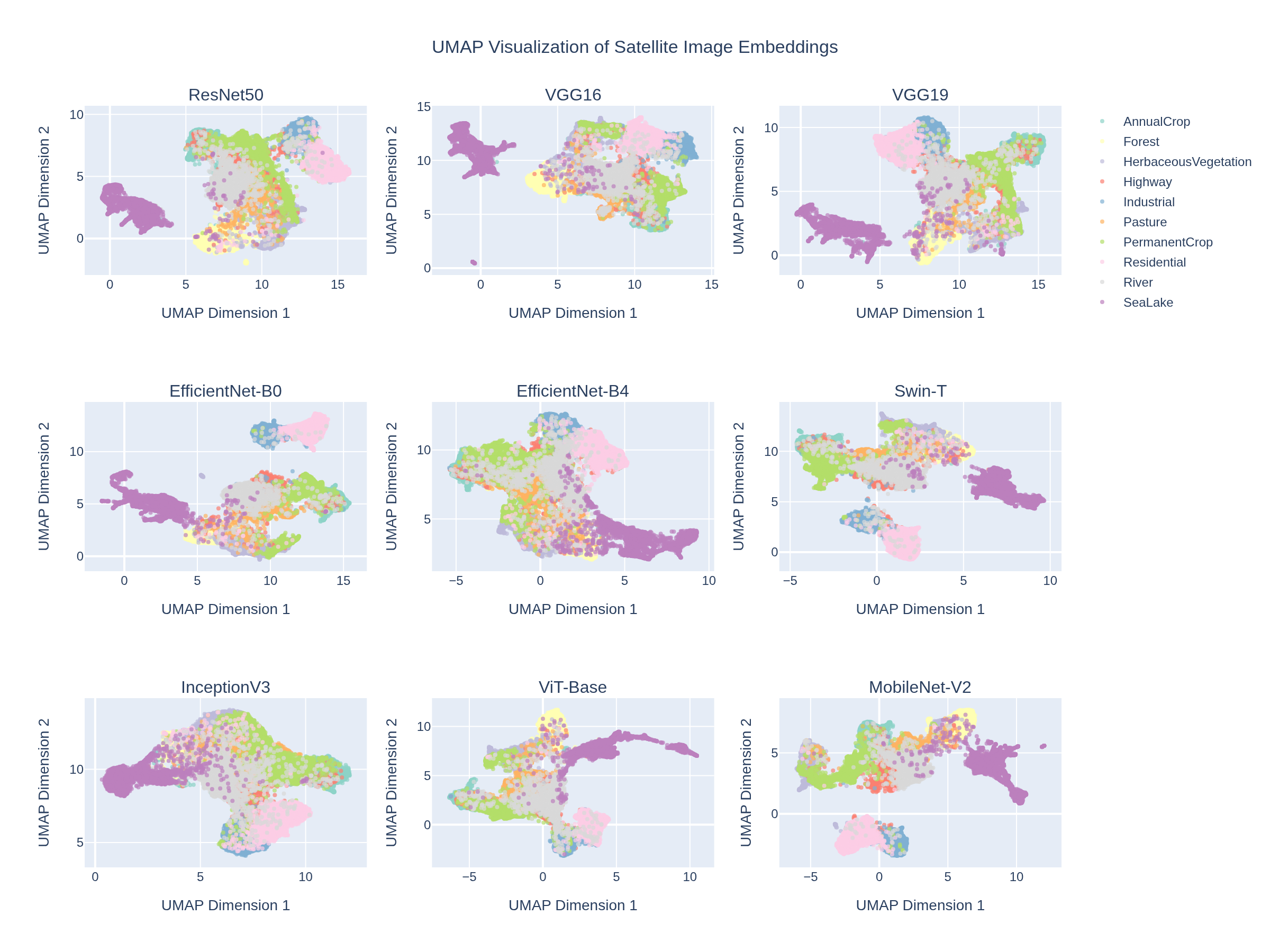}
\caption{UMAP visualization of EuroSAT satellite image features. Each point is an image colored by class. Several classes—AnnualCrop, Pasture, SeaLake, and River—show clearly separated clusters.}
\label{fig:eurosat_overview}
\end{minipage}
\begin{minipage}{\textwidth}
\centering
\includegraphics[width=\textwidth]{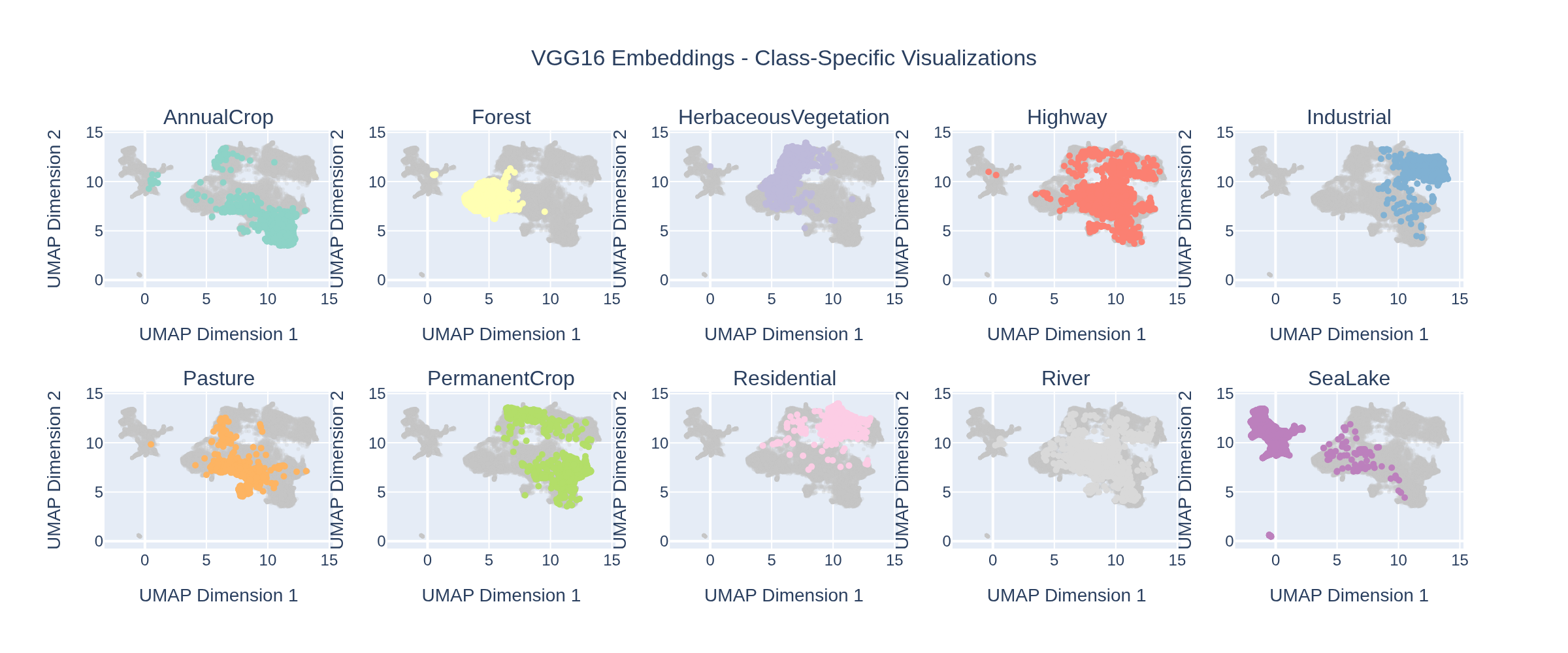}
\caption{VGG16 UMAP embeddings by class. Each panel highlights one EuroSAT class against others.}
\label{fig:vgg16_eurosat}
\end{minipage}
\end{figure}

The EuroSAT dataset \citep{helber2019eurosat} contains satellite images from 10 different land use and land cover classes: Annual Crop, Forest, Herbaceous Vegetation, Highway, Industrial, Pasture, Permanent Crop, Residential, River, and Sea Lake.

We extracted feature representations using a pre-trained VGG16 model and applied UMAP \citep{McInnes2018} dimensionality reduction to visualize the data structure. Figure~\ref{fig:eurosat_overview} shows the 2D UMAP embeddings for all classes across multiple pre-trained models, revealing that even within individual classes, the data often exhibits multimodal structure with distinct clusters.

\begin{figure}[htbp]
\centering
\includegraphics[width=\textwidth]{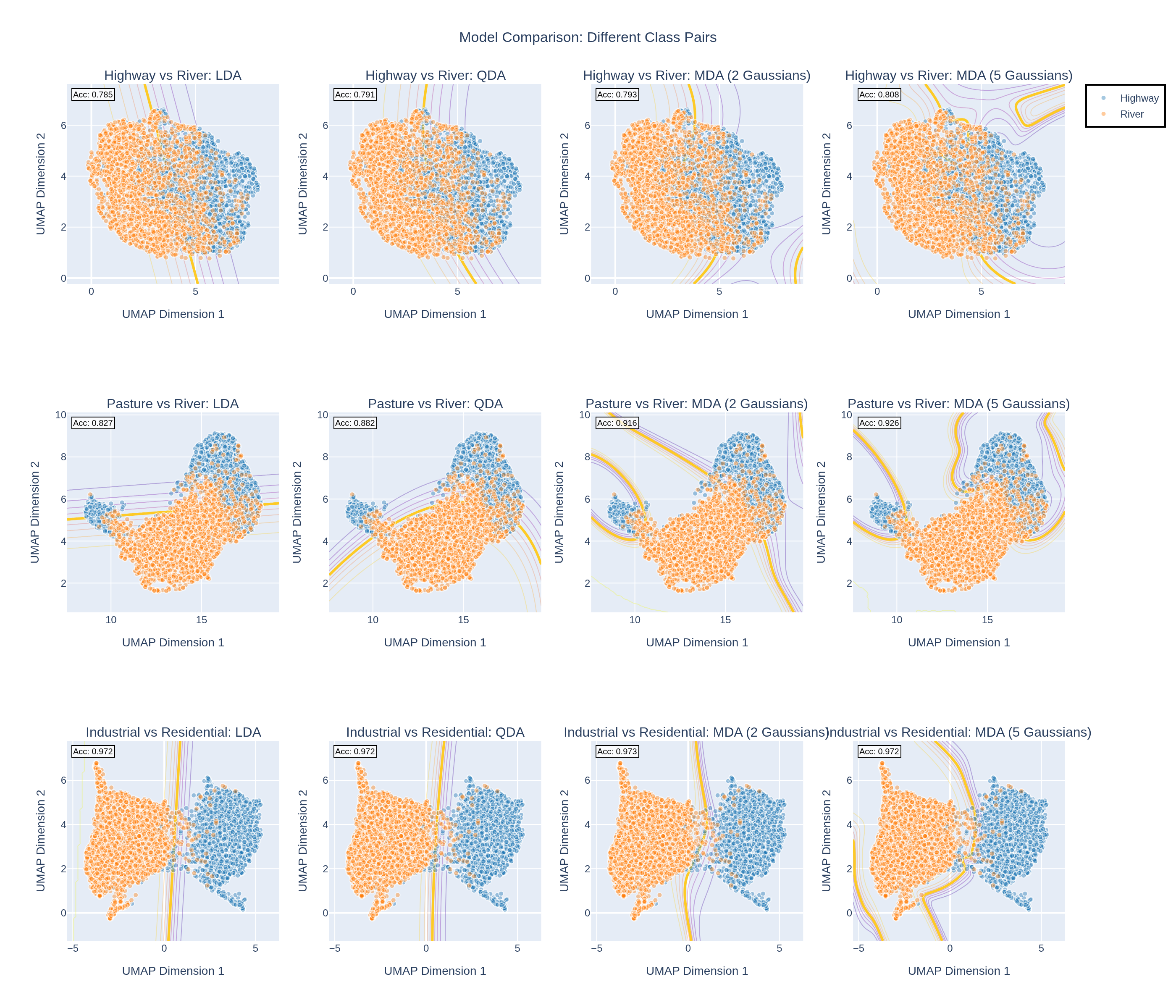}
\caption{Comparison of LDA and MDA with different numbers of components on satellite image classification tasks. Each row shows a different class pair (Highway vs River, Pasture vs River, Industrial vs Residential). Columns show decision boundaries for LDA, QDA, and MDA with 2 and 5 components per class. Points are colored by true class labels, and contour lines indicate decision boundaries. Classification accuracies are displayed in each subplot. Overspecified MDA with 5 components consistently achieves better separation of multimodal class distributions.}
\label{fig:mda_lda_comparison}
\end{figure}

To specifically test our theoretical predictions about overspecified MDA, we focused on binary classification tasks using three challenging class pairs: Highway vs River, Pasture vs River, and Industrial vs Residential. These pairs were selected as they exhibit varying degrees of class separability and internal cluster structure in the feature space.

For each class pair, we compared Linear Discriminant Analysis (LDA) with Quadratic Discriminant Analysis (QDA) and MDA with 2 and 5 Gaussian components per class. Figure~\ref{fig:mda_lda_comparison} illustrates the decision boundaries and classification performance. The results demonstrate that overspecification does not degrade performance: MDA with 2 and 5 components per class performs comparably to LDA and QDA in cases with a single cluster per class (e.g., Industrial vs. Residential). In contrast, MDA clearly outperforms LDA and QDA when multiple clusters per class are present (e.g., Pasture vs. River).

\section{Conclusion}
In this work, we provided a rigorous theoretical analysis of Mixture Discriminant Analysis (MDA) in the overspecified regime, where the number of Gaussian components used in model fitting exceeds that in the true data-generating process. Focusing on the practically relevant case of fitting an unbalanced two-component Gaussian mixture to data from a single Gaussian per class, we established sharp bounds on the excess classification error over the Bayes risk.  

At the population level, we proved that, under suitable initialization, the Expectation–Maximization (EM) algorithm achieves exponential convergence of the classification error to the Bayes risk, requiring only ${O}(\log(1/\epsilon))$ iterations to reach accuracy $\epsilon$. In the finite-sample setting, we derived non-asymptotic high-probability bounds showing that the misclassification error converges to the Bayes risk at the optimal ${O}(n^{-1/2})$ rate, with only ${O}(\log(n/d))$ EM iterations needed. The analysis hinges on controlling the Kullback–Leibler divergence between the fitted and true class-conditional densities, leveraging a Polyak–Łojasiewicz inequality on a hypersurface where the EM iterates reside.

Our results extend and refine existing theory for overspecified mixtures by incorporating both learned variance parameters and unbalanced component weights, leading to faster algorithmic convergence and optimal statistical rates. This provides a principled justification for the empirical success of overspecified MDA in complex classification problems, such as those arising in image or text analysis.  

\section*{Declaration} 
\paragraph{Conflicts of interest} The authors declare no conflicts of interest. 
\paragraph{Ethical approval} This study does not involve human participants. 
\paragraph{Informed consent} Not applicable.


\bibliography{ref}

\appendix

\section{Deferred Proofs}\label{app:proofs}

\subsection{Population EM Updates}\label{app:EM_upd_proof}
We first give more details on the EM algorithm. It will be convenient to represent the mixture distribution \eqref{eq:gmm} using a hidden Bernoulli random variable $K$, which serves as an identifier of the mixture components. Since the components have weights $p$ and $1-p$, we assume that $\Pr[K=0]=1-p$ and $\Pr[K=1]=p$. Next, we define the conditional distribution
$$
(\mathbf{X}\mid K=0)\sim\mathcal{N}(-\bth,\sigma^2\mathbf{I}),\qquad(\mathbf{X}\mid K=1)\sim\mathcal{N}(\bth,\sigma^2\mathbf{I}).
$$
This determines the joint distribution of the tuple $(\mathbf{X}, K)$, and by construction, the marginal distribution of $\mathbf{X}$ is the Gaussian mixture $\mathcal{G}(\bth,\sigma^2)$ given by \eqref{eq:gmm}. The Population EM algorithm attempts to maximize the expected log-likelihood \eqref{eq:e_log_l} by iteratively applying these two steps:
\begin{itemize}
    \item \emph{E step:} Given the current estimate $(\bth_t,\sigma^2_t)$, do a soft assignment of any $\mathbf{x}\in\mathbb{R}^d$ to the component $K=1$, i.e. compute the conditional probability of $K=1$ given $\mathbf{X}=\mathbf{x}$:
    \begin{equation}
    w(\mathbf{x};\bth_t,\sigma^2_t)=\frac{p\cdot\phi\left(\frac{\mathbf{x}-\bth_t}{\sigma_t}\right)}{p\cdot\phi\left(\frac{\mathbf{x}-\bth_t}{\sigma_t}\right)+(1-p)\cdot\phi\left(\frac{\mathbf{x}+\bth_t}{\sigma_t}\right)},\label{eq:weight_fn}  
    \end{equation}
    and use it to compute the $Q$-function:
    \begin{multline*}
        Q(\bth,\sigma^2;\bth_t,\sigma^2_t)
        :=\E_{\mathbf{Z}\sim\mathcal{N}(\mathbf{0},\mathbf{I}}\left[w(\mathbf{Z};\bth_t,\sigma^2_t)\log \left(p \phi\left(\frac{\mathbf{Z}-\bth}{\sigma}\right)\right)\right.\\\left.
        +(1-w(\mathbf{Z};\bth_t,\sigma^2_t))\log\left((1-p)\phi\left(\frac{\mathbf{Z}+\bth}{\sigma}\right)\right)\right]
    \end{multline*}
   
    \item \emph{M step:} Update the parameters by solving the optimization problem:
\begin{align}
(\bth_{t+1},\sigma^2_{t+1})\in\arg\max_{(\bth,\sigma^2)}Q(\bth,\sigma^2;\bth_t,\sigma^2_t).\label{eq:m_step}
\end{align}
\end{itemize}


\begin{lemma} \label{lem:EM_upd}
    Let the expected log-likelihood \eqref{eq:e_log_l}
    be maximized by the Population EM algorithm. Then the parameter updates are given by \eqref{eq:theta_upd} and \eqref{eq:sigma_sq_upd}, i.e.
    \begin{align*}
        \bth_{t+1}&=\E_{\mathbf{Z}\sim\mathcal{N}(\mathbf{0},\mathbf{I})}\left[t_p\left(\frac{\bth_t^\top\mathbf{Z}}{1-\frac{\|\bth_t\|^2}{d}}\right)\mathbf{Z}\right],\\
        \sigma^2_{t+1}&=1-\frac{\|\bth_{t+1}\|^2}{d},
    \end{align*}
    where $t_p(x):=\frac{p\cdot e^{x}-(1-p)\cdot e^{-x}}{p\cdot e^{x}+(1-p)\cdot e^{-x}}$.
\end{lemma}

\begin{proof}
    First, we rewrite the weight function \eqref{eq:weight_fn} as
    $$
    w(\mathbf{x};\bth,\sigma^2)=\frac{p}{p+(1-p)\exp\left(-\frac{2\bth^\top\mathbf{x}}{\sigma^2}\right)}=s_p\left(\frac{2\bth^\top\mathbf{x}}{\sigma^2}\right),
    $$
    where $s_p(x):=\frac{p}{p+(1-p)e^{-x}}$.
    Then the gradient of $Q$ w.r.t. $\bth$ is
    \begin{align*}
    \nabla_{\bth} Q&=\frac{1}{\sigma^2}\E_{\mathbf{Z}}\left[s_p\left(\frac{2\bth^\top_t\mathbf{Z}}{\sigma^2}\right)(\mathbf{Z}-\bth)-\left(1-s_p\left(\frac{2\bth^\top_t\mathbf{Z}}{\sigma^2}\right)\right)(\mathbf{Z}+\bth)\right]\\
    &=\frac{1}{\sigma^2}\E_\mathbf{Z}\left[\left(2s_p\left(\frac{2\bth^\top_t\mathbf{Z}}{\sigma^2}\right)-1\right)\mathbf{Z}-\bth\right]
    \end{align*}
    Solving $\nabla_{\bth} Q=\mathbf{0}$ for $\bth$ we get the population EM update for the location parameter:
    \begin{equation}
        \bth_{t+1}=\E_\mathbf{Z}\left[\left(2s_p\left(\frac{2\bth^\top_t\mathbf{Z}}{\sigma^2}\right)-1\right)\mathbf{Z}\right]=\E_{\mathbf{Z}}\left[t_p\left(\frac{\bth^\top_t\mathbf{Z}}{\sigma^2_t}\right)\mathbf{Z}\right],\label{eq:theta_upd_raw}
    \end{equation}
    where we denoted $t_p(x):=2s_p(2x)-1$.   Similarly, for $\sigma^2$ we have
    \begin{align*}
        \frac{\partial Q}{\partial\sigma^2}&=\E_{\mathbf{Z}}\left[\frac{d}{\sigma^2}-s_p\left(\frac{2\bth^\top_t\mathbf{Z}}{\sigma^2}\right)\frac{\|\mathbf{Z}-\bth\|^2}{2\sigma^4}-\left(1-s_p\left(\frac{2\bth^\top_t\mathbf{Z}}{\sigma^2}\right)\right)\frac{\|\mathbf{Z}+\bth\|^2}{2\sigma^4}\right]\\
        &=\frac{d}{2\sigma^2}-\frac{1}{2\sigma^4}\E_{\mathbf{Z}}\left[\|\mathbf{Z}\|^2-2\left(2s_p\left(\frac{2\bth^\top_t\mathbf{Z}}{\sigma^2}\right)-1\right)\bth^\top\mathbf{Z}+\|\bth\|^2\right]\\
        &=\frac{d}{2\sigma^2}-\frac{1}{2\sigma^4}\left(d-2\bth^\top\bth_{t+1}+\|\bth\|^2\right)
    \end{align*}
    Solving $\frac{\partial Q}{\partial\sigma^2}=0$ for $\sigma^2$ and then substituting $\bth=\bth_{t+1}$ we get the population EM update for the scale parameter:
    $$
    \sigma^2_{t+1}=1-\frac{\|\bth_{t+1}\|^2}{d},
    $$
    which is exactly formula \eqref{eq:sigma_sq_upd}. Now, plugging $\sigma^2_t=1-\frac{\|\bth_t\|^2}{d}$ into \eqref{eq:theta_upd_raw} we obtain \eqref{eq:theta_upd}.
\end{proof}

\subsection{Radiality of the function $L(\theta)$}\label{app:ell_radial}

\begin{lemma} \label{lem:ell_radial} Consider the function $L:\mathbb{R}^d\to\mathbb{R}$ defined as
\begin{equation}
    L(\bth):=-\mathcal{L}(\bth,1-\|\bth\|^2/d).\notag  
\end{equation}
Then $L(\bth)$ is a radial function of $\bth\in\mathbb{R}^d$.     It can be explicitly written as
    \begin{equation*}
    L(\bth)=\frac{d}2\log(2\pi(1-\|\bth\|^2/d))+\frac{d+\|\bth\|^2}{2(1-\|\bth\|^2/d)}-\E_{Z\sim\mathcal{N}(0,1)}\left[\log\left(c_p\left(\frac{\|\bth\|Z}{1-\|\bth\|^2/d}\right)\right)\right],    
    \end{equation*}
    where $c_p(x):=p\cdot e^x+(1-p)\cdot e^{-x}$.
\end{lemma}

\begin{proof}
    Since $f(\mathbf{x};\bth,\sigma^2)$ is the p.d.f. of $(1-p)\cdot\mathcal{N}(-\bth,\sigma^2\mathbf{I})+p\cdot\mathcal{N}(+\bth,\sigma^2\mathbf{I})$, it can be written as
    \begin{align*}
    f(\mathbf{x};\bth,\sigma^2)&=\frac{1-p}{\sigma^d}\phi\left(\frac{\mathbf{x}+\bth}{\sigma}\right)+\frac{p}{\sigma^d}\phi\left(\frac{\mathbf{x}-\bth}{\sigma}\right)\\
    &=(2\pi\sigma^2)^{-d/2}\cdot\exp\left(-\frac{\|\mathbf{x}\|^2+\|\bth\|^2}{2\sigma^2}\right)\cdot c_p\left(\frac{\bth^\top\mathbf{x}}{\sigma^2}\right).
    \end{align*}
    Hence
    $$
        \log f(\mathbf{x};\bth,\sigma^2)=-\frac{d}{2}\log(2\pi\sigma^2)-\frac{\|\mathbf{x}\|^2+\|\bth\|^2}{2\sigma^2}+\log\left(c_p\left(\frac{\bth^\top\mathbf{x}}{\sigma^2}\right)\right),
    $$
    and thus the negative population log-likelihood is
    \begin{align}
        -\mathcal{L}(\bth,\sigma^2)&=-{\E_{\mathbf{Z}}\left[\log f(\mathbf{Z};\bth,\sigma^2)\right]}\notag\\
        &=\frac{d}2\log(2\pi\sigma^2)+\frac{d+\|\bth\|^2}{2\sigma^2}-\E_{\mathbf{Z}}\left[\log\left(c_p\left(\frac{\boldsymbol\theta^\top\mathbf{Z}}{\sigma^2}\right)\right)\right]\notag\\
        &=\{\text{since $\bth^\top\mathbf{Z}\sim\mathcal{N}(0,\|\bth\|^2)$}\}\notag\\
        &=\underbrace{\frac{d}2\log(2\pi\sigma^2)+\frac{d+\|\bth\|^2}{2\sigma^2}-\E_{Z\sim\mathcal{N}(0,1)}\left[\log\left(c_p\left(\frac{\|\bth\|Z}{\sigma^2}\right)\right)\right]}_{q(\|\bth\|,\sigma^2)}.\label{eq:q_def}
    \end{align}  
As we can notice, the $\mathcal{L}(\bth,\sigma^2)$ depends on $\bth$ only through its norm. Plugging  in $\sigma^2=1-\|\bth\|^2/d$, we get
\begin{align*}
L(\bth)&=-\mathcal{L}(\bth,1-\|\bth\|^2/d)\\
&=\frac{d}2\log(2\pi(1-\|\bth\|^2/d))+\frac{d+\|\bth\|^2}{2(1-\|\bth\|^2/d)}\\&\quad-\E_{Z\sim\mathcal{N}(0,1)}\left[\log\left(c_p\left(\frac{\|\bth\|Z}{1-\|\bth\|^2/d}\right)\right)\right],
\end{align*}
which immediately implies the statement of the lemma. 
\end{proof}

\subsection{Radiality of $\|M(\bth)\|$}\label{app:m_radial}

\begin{lemma} \label{lem:em_norm} For the population EM operator $M(\bth)$ defined by \eqref{eq:pop_em_op}, we have
$$
\|M(\bth)\|=\E_{Z\sim\mathcal{N}(0,1)}\left[t_p\left(\frac{\|\bth\|Z}{1-\frac{\|\bth\|^2}{d}}\right)Z\right].
$$
\end{lemma}

\begin{proof}
    Let $\mathbf{R}$ be an orthonormal matrix such that $\mathbf{R}\bth=\|\bth\|\mathbf{e}_1$, where $\mathbf{e}_1$ is the first canonical basis vector in $\mathbb{R}^d$. Let $\mathbf{Y}=\mathbf{RZ}$, then $\mathbf{Y}\sim\mathcal{N}(\mathbf{0},\mathbf{I})$, and $\mathbf{Z}=\mathbf{R}^\top\mathbf{Y}$. Thus, we have
    \begin{align*}
        \|M(\bth)\|&=\left\|\E_{\mathbf{Z}\sim\mathcal{N}(\mathbf{0},\mathbf{I})}\left[t_p\left(\frac{\bth^\top\mathbf{Z}}{1-\sfrac{\|\bth\|^2}d}\right)\mathbf{Z}\right]\right\|\\&=\left\|\E_{\mathbf{Y}\sim\mathcal{N}(\mathbf{0},\mathbf{I})}\left[t_p\left(\frac{\|\bth\|Y_1}{1-\sfrac{\|\bth\|^2}d}\right)\mathbf{R}^\top\mathbf{Y}\right]\right\|\\
        &=\left\|\mathbf{R}^\top\E_{\mathbf{Y}\sim\mathcal{N}(\mathbf{0},\mathbf{I})}\left[t_p\left(\frac{\|\bth\|Y_1}{1-\sfrac{\|\bth\|^2}d}\right)\cdot\begin{bmatrix}Y_1\\\vdots\\ Y_n\end{bmatrix}\right]\right\|\\
        &=\left\|\mathbf{R}^\top\begin{bmatrix}\E_{Y_1\sim\mathcal{N}({0},1)}\left[t_p\left(\frac{\|\bth\|Y_1}{1-\sfrac{\|\bth\|^2}d}\right)Y_1\right]\\0\\\vdots\\ 0\end{bmatrix}\right\|\\
        &=\E_{Y_1\sim\mathcal{N}({0},1)}\left[t_p\left(\frac{\|\bth\|Y_1}{1-\sfrac{\|\bth\|^2}d}\right)Y_1\right].
    \end{align*}
\end{proof}

\subsection{Proof of Lemma~\ref{lem:properties}}

We begin with an auxiliary lemma.
\begin{lemma}
    \label{lem:M_prime} For any $\theta>0$,
        \begin{equation}
    m'(\theta)\le\frac{1+\theta^2/d}{(1-\theta^2/d)^2}\cdot\left(1-\frac{(2p-1)^2}{2}\right).\label{eq:m_prime}
    \end{equation}
\end{lemma}
\begin{proof}
    \label{app:lem_m_prime_proof}
    By direct calculation,
    \begin{align}
    &m'(\theta)=\frac{d}{d\theta}\E_{Z\sim\mathcal{N}(0,1)}\left[t_p\left(\frac{\theta Z}{1-\theta^2/d}\right)Z\right]\notag\\
    &=\frac{1+\theta^2/d}{(1-\theta^2/d)^2}\cdot\E_{Z\sim\mathcal{N}(0,1)}\left[\frac{4p(1-p)Z^2}{\left(p\cdot\exp\left(\frac{\theta Z}{1-\theta^2/d}\right)+(1-p)\cdot\exp\left(-\frac{\theta Z}{1-\theta^2/d}\right)\right)^2}\right]
    \end{align}
    Observe that the function $c_p(x):= p\cdot e^x+(1-p)\cdot e^{-x}$ is convex, has the minimum value $g_p\left(\log\frac{\sqrt{p(1-p)}}{p}\right)=2\sqrt{p(1-p)}$, and  satisfies $g_p\left(\log\frac{1-p}{p}\right)=g_p(0)=1$ assuming $p\in(1/2,1]$. Thus
\begin{align}
    &c_p(x)\in\left[2\sqrt{p(1-p)},1\right]\quad\text{for } x\in\left[\log\frac{1-p}{p},0\right],\notag\\
    &c_p(x)\in(1,+\infty)\quad\text{for }x\notin\left[\log\frac{1-p}{p},0\right].\label{eq:c_p_bounds}
\end{align}
Defining the event
$$
\mathcal{A}:=\left\{\frac{\theta Z}{1-\theta^2/d}\in\left[\log\frac{1-p}{p},0\right]\right\},
$$
we have
\begin{align}
    &\E_{Z\sim\mathcal{N}(0,1)}\left[\frac{Z^2}{\left(p\cdot\exp\left(\frac{\theta Z}{1-\theta^2/d}\right)+(1-p)\cdot\exp\left(-\frac{\theta Z}{1-\theta^2/d}\right)\right)^2}\right]
    \notag\\
    &\le\frac{1}{4p(1-p)}\cdot\E_{Z\sim\mathcal{N}(0,1)}\left[Z^2\mathbb{I}(\mathcal{A})\right]+\E_{Z\sim\mathcal{N}(0,1)}\left[Z^2\mathbb{I}(\mathcal{A}^c)\right]\notag\\
    &=\E_{Z\sim\mathcal{N}(0,1)}[Z^2]+\left(\frac{1}{4p(1-p)}-1\right)\E_{Z\sim\mathcal{N}(0,1)}\left[Z^2\mathbb{I}(\mathcal{A})\right]\notag\\
    &\le1+\frac{(2p-1)^2}{4p(1-p)}\E_{Z\sim\mathcal{N}(0,1)}\left[Z^2\mathbb{I}(Z\le0)\right]=1+\frac{(2p-1)^2}{8p(1-p)},\label{eq:m_prime_bound}
\end{align}
where we used the fact $\mathcal{A}\subset\{Z\le0\}$. Thus, from \eqref{eq:m_mvt} and \eqref{eq:m_prime_bound} we have
\begin{align*}
    m'(\theta)&\le\frac{1+\theta^2/d}{(1-\theta^2/d)^2}\cdot\left(4p(1-p)+\frac{(2p-1)^2}{2}\right)\notag\\
    &=\frac{1+\theta^2/d}{(1-\theta^2/d)^2}\cdot\left(1-\frac{(2p-1)^2}{2}\right).
\end{align*}
\end{proof}

\begin{proof}[Proof of Lemma~\ref{lem:properties}]
\begin{enumerate}

    \item According to Lemma~\ref{lem:M_prime}, for the inequality $m'(\theta)<1$ to hold it is sufficient that $\theta$ satisfies
$$
\frac{1+\theta^2/d}{(1-\theta^2/d)^2}\cdot\left(1-\frac{(2p-1)^2}{2}\right)<1
$$
Solving the latter inequality for $\theta>0$ we get
$$
\theta<\sqrt{d\cdot\frac{2+q-\sqrt{8q+q^2}}{2}},
$$
where $q:=1-\frac{(2p-1)^2}{2}\in[1/2,1)$ for $p\ne\sfrac12$.

\item By the mean value theorem, there exists $\xi\in[0,\theta]$ such that
\begin{equation}
    m(\theta)=m(\theta)-m(0)=m'(\xi)\cdot\theta\label{eq:m_mvt}
\end{equation}
From Lemma~\ref{lem:M_prime} and the fact that $\xi\le\theta$, we have
\begin{align*}
m'(\xi)&\le\frac{1+\xi^2/d}{(1-\xi^2/d)^2}\cdot\left(1-\frac{(2p-1)^2}{2}\right)\\
&\le\frac{1+\theta^2/d}{(1-\theta^2/d)^2}\cdot\left(1-\frac{(2p-1)^2}{2}\right).
\end{align*}
Now denote $\rho:=\frac{1+\theta_0^2/d}{(1-\theta_0^2/d)^2}\cdot\left(1-\frac{(2p-1)^2}{2}\right)$. From the condition $\theta_0<\sqrt{d\cdot\frac{2+q-\sqrt{8q+q^2}}{2}}$ it follows that $\rho<1$. Thus,
\begin{equation}
m'(\xi)\le \rho<1,\label{eq:rho_ineq}
\end{equation}
and the corollary follows from \eqref{eq:m_mvt} and \eqref{eq:rho_ineq}.

    \item Recall the notation $\theta:=\|\bth\|$ and $\ell(\theta):=L(\bth)$. Let  $q(\theta,\sigma^2)$ be the function defined in \eqref{eq:q_def}, then $\ell(\theta):=q(\theta,1-\theta^2/d)$. To analyze the derivative $\ell'$, we first find partial derivatives of $q$:

    \begin{align*}
        \frac{\partial q}{\partial\theta}(\theta,\sigma^2)&=\frac{\theta}{\sigma^2}-\frac{1}{\sigma^2}\E_{{Z}\sim\mathcal{N}(0,1)}\left[t_p\left(\frac{\theta{Z}}{\sigma^2}\right){Z}\right],\\
        \frac{\partial q}{\partial\sigma^2}(\theta,\sigma^2)&=\frac{d}{2\sigma^2}-\frac{d+\theta^2}{2\sigma^4}+\frac{1}{\sigma^4}\E_{{Z}\sim\mathcal{N}(0,1)}\left[t_p\left(\frac{\theta{Z}}{\sigma^2}\right)\theta{Z}\right]
    \end{align*}
    Using the notation $m(\theta):=\E_{{Z}\sim\mathcal{N}(0,1)}\left[t_p\left(\frac{\theta{Z}}{1-\theta^2/d}\right){Z}\right]$ introduced in \eqref{eq:m_fun}, we have
    \begin{align}
        \ell'(\theta)&=\frac{\partial q}{\partial\theta}(\theta,1-\theta^2/d)+\frac{\partial q}{\partial\sigma^2}(\theta,1-\theta^2/d)\cdot\left(-\frac{2\theta}{d}\right)\notag\\
        &=\frac{\theta}{1-\theta^2/d}-\frac{m(\theta)}{1-\theta^2/d}-\frac{\theta}{1-\theta^2/d}+\frac{(1+\theta^2/d)\theta}{(1-\theta^2/d)^2}-\frac{2m(\theta)\cdot\theta^2/d}{(1-\theta^2/d)^2}\notag\\
        &=\frac{-m(\theta)(1-\theta^2/d)+(1+\theta^2/d)\theta-2m(\theta)\theta^2/d}{(1-\theta^2/d)^2}\notag\\
        &=\frac{1+\theta^2/d}{(1-\theta^2/d)^2}\cdot\left(\theta-m(\theta)\right).\label{eq:lambda_prime}
    \end{align}
    Then for $\theta$ in the interval mentioned in the statement of Part~\ref{lem:M_bounds},
    $$
    m'(\theta)< 1.
    $$
    Thus, the function $\theta\mapsto\theta-m(\theta)$ has a positive derivative on $[0,\theta_0]$ and therefore it is  increasing there. Now, $\ell'(\theta)$ is itself increasing on this interval as the product of two positive increasing functions. This implies that for $\theta\in[0,\theta_0]$, $\ell(\theta)$ is convex.
    \item     From Lemma~\ref{lem:properties} Part~\ref{lem:M_bounds} and \eqref{eq:lambda_prime}, we have
    \begin{equation}
        \ell'(\theta)\ge\frac{1+\theta^2/d}{(1-\theta^2/d)^2}\cdot(1-\rho)\cdot\theta,\label{eq:pl_left}
    \end{equation}
    where $\rho:=\frac{1+\theta_0^2/d}{(1-\theta_0^2/d)^2}\cdot\left(1-\frac{(2p-1)^2}{2}\right)\in(0,1)$. Notice that $\ell'(\theta)>0$. Thus, multiplying both sides of \eqref{eq:pl_left} by $\ell'(\theta)$, we have
    $$
    [\ell'(\theta)]^2\ge\frac{1+\theta^2/d}{(1-\theta^2/d)^2}\cdot({1-\rho})\cdot\ell'(\theta)\cdot\theta\ge(1-\rho)\cdot[\ell(\theta)-\ell(0)],
    $$
    where we used convexity of $\ell$ (Part~\ref{lem:convex}).
    \end{enumerate}
\end{proof}

\subsection{Sample-based EM updates}\label{app:sample_em_upd}

\begin{lemma}
    Let the log-likelihood \eqref{eq:sample_log_l}
    be maximized by the EM algorithm. Then the parameter updates are given by
    \begin{align}
        \hat{\bth}_{t+1}&=\frac1n\sum_{i=1}^n t_p\left(\frac{\hat\bth_t^\top\mathbf{Z}_i}{\sfrac{\sum_{i=1}^n\|\mathbf{Z}_i\|^2}{nd}-\sfrac{\|\hat\bth_t\|^2}{d}}\right)\mathbf{Z}_i,\notag\\
        \hat{\sigma}^2_{t+1}&=\frac{\sum_{i=1}^n\|\mathbf{Z}_i\|^2}{nd}-\frac{\|\hat\bth_{t+1}\|^2}{d}.\notag
    \end{align}
\end{lemma}
\begin{proof}
    See \citet{DBLP:conf/aistats/DwivediHKWJ020}.
\end{proof}

\subsection{Perturbation bound}\label{app:perturb}
\begin{lemma}\label{lem:perturb}
There exists a universal constant $c$ such that for any fixed $\delta\in(0,1)$, $\alpha\in(0,\sfrac12)$, and $r\in[0,\theta_0]$ with $\theta_0<\sqrt{d\cdot\frac{2+q-\sqrt{8q+q^2}}{2}}$ where $q=1-\frac{(2p-1)^2}{2}$, we have
$$
    \Pr\left[\sup_{\theta\in[0,r]}|{m}_n(\theta)-{m}(\theta)|\le cr\sqrt{\frac{\log(1/\delta)}{n}}\right]\ge1-\delta.
    $$    
    for $n\succsim\log^{\frac{1}{2\alpha}}\frac1\delta$.
\end{lemma}
\begin{proof}The statement follows from the following two lemmas.\end{proof}
\begin{lemma}\label{lem:pseudo_perturb}
Consider a function
\begin{equation}
    \widetilde{m}_n(\theta):=\E_{Z\sim\mathcal{N}(0,1)}\left[t_p\left(\frac{\theta Z}{\sfrac{U}{nd}-\sfrac{\theta^2}{d}}\right)Z\right],\label{eq:pseudo_em}
\end{equation}
where $U\sim\chi^2_{nd}$.
    There exists a universal constant $c$ such that for any fixed $\delta\in(0,1)$, $\alpha\in(0,\sfrac12)$, and $r\in[0,\theta_0]$ with $\theta_0<\sqrt{d\cdot\frac{2+q-\sqrt{8q+q^2}}{2}}$ where $q=1-\frac{(2p-1)^2}{2}$, we have
$$
    \Pr\left[\sup_{\theta\in[0,r]}|{m}_n(\theta)-\widetilde{m}_n(\theta)|\le cr\sqrt{\frac{\log(1/\delta)}{n}}\right]\ge1-\delta-2e^{-\sfrac{(nd)^{2\alpha}}8}.
    $$
\end{lemma}
\begin{proof}
    Denote $\mathcal{I}:=[1-(nd)^{-\sfrac12+\alpha}-\sfrac{\theta_0^2}{d}, 1-(nd)^{-\sfrac12+\alpha}]$. Fix $r\in[0,\theta_0]$ and define $\widetilde{r}:=\frac{r}{1-(nd)^{-\sfrac12+\alpha}-\theta_0^2/d}$. For sufficiently large $n$, we have $\widetilde{r}\le(1+\sfrac{2\theta_0^2}d) r$. Define the event 
    $$
    \mathcal{A}_\alpha:=\left\{\left|\frac{U}{nd}-1\right|\le\frac{1}{(nd)^{\frac12-\alpha}}\right\}
    $$
    and notice that $\Pr[\mathcal{A}_\alpha]\ge1-2e^{-\sfrac{(nd)^{2\alpha}}8}$ from standard chi-squared tail bounds. Assuming the event $\mathcal{A}_\alpha$ holds, we have
    \begin{align*}
        \left|m_n(\theta)-\widetilde{m}_n(\theta)\right|&\le\sup_{\theta\in[0,r],\sigma^2\in\mathcal{I}}\left|\frac1n\sum_{i=1}^n t_p\left(\frac{\theta Z_i}{\sigma^2}\right)Z_i-\E_{Z\sim\mathcal{N}(0,1)}\left[t_p\left(\frac{\theta Z}{\sigma^2}\right)Z\right]\right|\\
        &\le\sup_{\tilde\theta\in[0,\widetilde{r}]}\left|\mu_n(\tilde\theta)-\mu(\tilde\theta)\right|,
    \end{align*}
    for any $\theta\in[0,r]$, where the functions $\mu_n$ and $\mu$ are defined as
    $$
    \mu_n(\tilde\theta):=\frac1n\sum_{i=1}^n t_p(\tilde\theta Z_i) Z_i\qquad\text{and}\qquad \mu(\tilde\theta):=\E_{Z\sim\mathcal{N}(0,1)}[t_p(\tilde\theta Z)Z]
    $$
    Define the random variable
    $$
    X:=\sup_{\tilde\theta\in[0,\widetilde{r}]}\left|\mu_n(\tilde\theta)-\mu(\tilde\theta)\right|
    $$
    By the symmetrization lemma
    \begin{equation}
    \E\left[\exp(\lambda X)\right]\le\E\left[\exp\left(\sup_{\tilde\theta\in[0,\widetilde{r}]}\frac{2\lambda}n\sum_{i=1}^n\epsilon_i t_p(\tilde\theta Z_i)Z_i\right)\right],\label{eq:symm_one}
    \end{equation}
    where $\{\epsilon_i\}_{i=1}^n$ are i.i.d. Rademacher random variables independent of $\{Z_i\}_{i=1}^n$. Note that
    $$
    c_p(x)\ge2\sqrt{p(1-p)}\quad\Rightarrow\quad t'_p(x)=\frac{4p(1-p)}{[c_p(x)]^2}\le1\qquad\text{(see Appendix~\ref{app:lem_m_prime_proof})},
    $$
    which implies 
    $$
    \left|t_p(\tilde\theta x)-t_p(\tilde\theta' x)\right|\le\left|(\tilde\theta-\tilde\theta')x\right|\qquad\text{for all } x.
    $$
    Thus, for any fixed $x$, the function $\tilde\theta\mapsto t_p(\tilde\theta x)$ is Lipschitz. Hence, by the Ledoux-Talagrand contraction principle
    \begin{align}
        &\E\left[\exp\left(\sup_{\tilde\theta\in[0,\widetilde{r}]}\frac{2\lambda}n\sum_{i=1}^n \epsilon_i t_p(\tilde\theta Z_i)Z_i\right)\right]\le\E\left[\exp\left(\sup_{\tilde\theta\in[0,\widetilde{r}]}\frac{4\lambda}{n}\sum_{i=1}^n\epsilon_i\tilde\theta Z_i^2\right)\right]\notag\\
    &\le\E\left[\exp\left(\frac{4\lambda \widetilde{r}}n\sum_{i=1}^n\epsilon_i Z_i^2\right)\right]\label{eq:symm_two}
    \end{align}
    Note that for a Rademacher r.v. $\epsilon$ and an independent standard normal r.v. $Z$ we have
    \begin{align*}
    \E[\exp(t\epsilon Z^2)]&=\frac12\E[\exp(tZ^2)]+\frac12\E[\exp(-tZ^2)]\\
    &=\frac{1}{2\sqrt{1-2t}}+\frac{1}{2\sqrt{1+2t}}\stackrel{(\bigstar)}{\le}1+2t^2\qquad\text{for }|t|\le\frac14,
    \end{align*}
    where the inequality $(\bigstar)$ will be proven later. Thus, for large enough $n$,
    \begin{equation}
    \E\left[\exp\left(\frac{4\lambda\widetilde{r}}{n}\sum_{i=1}^n\epsilon_i Z_i^2\right)\right]\le\left(1+\frac{32\lambda^2\widetilde{r}^2}{n^2}\right)^{n}\le\exp\left(\frac{32\lambda^2\widetilde{r}^2}{n}\right).\label{eq:symm_three}    
    \end{equation}
    From \eqref{eq:symm_one}, \eqref{eq:symm_two}, and \eqref{eq:symm_three}, we have
    $$
    \E[\exp(\lambda X)]\le\exp\left(\frac{32\lambda^2\widetilde{r}^2}{n}\right)
    $$
    Hence, by Chernoff bound,
    \begin{align*}
    \Pr[X\ge a]&\le\inf_{\lambda>0}\E[\exp(\lambda(X-a))]\le\inf_{\lambda>0}\left[\exp\left(\frac{32\lambda^2\widetilde{r}^2}n-\lambda a\right)\right]\\&=\exp\left(-\frac{a^2 n}{128\widetilde{r}^2}\right)=\delta.
    \end{align*}
    Solving the latter equation for $a$, we have (still conditionally on $\mathcal{A}_\alpha$)
    $$
    \Pr\left[X\le8\sqrt{2}\widetilde{r}\sqrt{\frac{\log(1/\delta)}{n}}\right]\ge1-\delta
    $$
    Finally, using $\widetilde{r}\le(1+\sfrac{2\theta_0^2}d) r$ for large $n$ and $\Pr[\mathcal{A}_\alpha]\ge1-2e^{-\sfrac{(nd)^{2\alpha}}8}$, we obtain that
    $$
    \Pr\left[\sup_{\theta\in[0,r]}\left|m_n(\theta)-\widetilde{m}_n(\theta)\right|\le c(\theta_0,d) r\sqrt{\frac{\log(1/\delta)}{n}}\right]\ge1-\delta-e^{-\sfrac{(nd)^{2\alpha}}8},
    $$
    where $c(\theta_0,d):=8\sqrt{2}\left(1+\sfrac{2\theta_0^2}{d}\right)$.
\end{proof}
    
\begin{proof}[Proof of inequality $(\bigstar)$]
Let $f(t):=(1-2t)^{-1/2}+(1+2t)^{-1/2}$ and $g(t):=2(1+2t^2)$. Then set $h(t):=(f(t)/g(t))-1.$ By simple calculus,
		\[h'(t):=\frac{(20t^3-2t^2-2t-1)\sqrt{1-2t}+(20t^3+2t^2-2t+1)\sqrt{1+2t}}{(1+2t)^{3/2}\,(1-2t)^{3/2} \, (8t^4+8t^2+2)}.\]
		On the domain $[0,1/4]$, the roots of the numerator can be found exactly, with the result that $h(t)$ has critical points only at $t=0$ and at $t_0 \approx 0.208...$.  By checking directly, $h(0)=0, h(1/4)<0,$ and $h(t_0)<0$.  We conclude that the absolute maximum value of $h(t)$ on $[0,1/4]$ is $0$. Since $h(t)$ is even, we conclude $h(t) \leq 0$ for all $|t|\leq1/4$. But since $f(t),g(t)$ are nonnegative, this is the same as $f(t)\leq g(t)$, which is the desired inequality.
\end{proof}

\begin{lemma}\label{lem:perturb2}
    Let $\widetilde{m}_n(\theta)$ be the function defined in \eqref{eq:pseudo_em}. Then there exists a universal constant $c$ such that for any fixed $\delta\in(0,1)$, $\alpha\in(0,\sfrac12)$, and $r\in[0,\theta_0]$ we have
    $$
    \Pr\left[\sup_{\theta\in[0,r]}\left|\widetilde{m}_n(\theta)-m(\theta)\right|\le cr\sqrt{\frac{\log(1/\delta)}{nd}}\right]\ge1-\delta
    $$
\end{lemma}

\begin{proof}
    Let $\mathcal{A}_\epsilon:=\{|U/(nd)-1|\le \epsilon\}$. By standard chi-squared tail bounds, $\Pr[\mathcal{A}_\epsilon]\ge1-2e^{-\epsilon^2 nd/8}$. Define a function
    $$
f(\theta,u):=\E_Z\left[t_p\left(\frac{\theta Z}{u-\theta^2/d}\right)Z\right].
$$
By the Mean Value Theorem, we have
    \begin{align*}
        f(\theta,U/(nd))-f(\theta,1)&=\frac{\partial f}{\partial u}(\theta,\xi)\cdot\left(\frac{U_{nd}}{nd}-1\right)\\
        &=-\E_Z\left[t'_p\left(\frac{\theta Z}{\xi-\theta^2/d}\right)Z^2\right]\cdot\frac{\theta}{(\xi-\theta^2/d)^2}\cdot\left(\frac{U_{nd}}{nd}-1\right),
    \end{align*}
    where $\xi$ is between $U/(nd)$ and $1$. Therefore, for $\theta\in[0,r]$, small enough $\epsilon$ and $\theta_0$, conditional on $\mathcal{A}_\epsilon$, we get
    $$
        |\underbrace{f(\theta,U_{nd}/(nd))}_{\widetilde{m}_n(\theta)}-\underbrace{f(\theta,1)}_{m(\theta)}|\le\frac{\theta}{(1-\epsilon-\theta^2/d)^2}\cdot\epsilon\le (1+2\epsilon+2\theta_0^2/d)r\epsilon.
    $$
    Solving $2e^{-\epsilon^2 nd/8}=\delta$ for $\epsilon$, we have
    $$
    \Pr\left[\sup_{\theta\in[0,r]}\left|\widetilde{m}_n(\theta)-m(\theta)\right|\le cr\sqrt{\frac{\log(1/\delta)}{nd}}\right]\ge1-\delta
    $$
\end{proof}

\subsection{Convergence of the iterates $\hat{\theta}_t$}\label{app:theta_convergence}
\begin{lemma}\label{lem:theta_convergence}
    Suppose that we fit an unbalanced mixture model \eqref{eq:gmm} to $\mathcal{N}(\mathbf{0},\mathbf{I})$. Then for any starting point ${\boldsymbol{\theta}}_0$ such that $\|\boldsymbol{\theta}_0\|<\sqrt{d\cdot\frac{2+q-\sqrt{8q+q^2}}{2}}$ where $q:=1-\frac{(2p-1)^2}{2}$,  the sequence of EM iterates $\hat{\boldsymbol{\theta}}_{t+1}=M_n(\hat{\boldsymbol{\theta}}_t)$ satisfies
    $$
    \|\hat{\boldsymbol{\theta}}_T\|\le \frac{c\|\boldsymbol{\theta}_0\|}{1-\rho}\sqrt{\frac{\log(1/\delta)}{n}}
    $$
    for $T\ge\log_{1/\rho}\left((1-\rho)\sqrt{\frac{n}{\log(1/\delta)}}\right)$ with probability at least $1-\delta$.
\end{lemma}
\begin{proof} Define the event 
$$
\mathcal{A}_\delta:=\left\{\sup_{\theta\in[0,\theta_0]}\left|m_n(\theta)-m(\theta)\right|\le c\theta_0\sqrt{\frac{\log(1/\delta)}{n}}\right\}
$$
and note that by Lemma~\ref{lem:perturb}, $\Pr[\mathcal{A}_\delta]\ge1-\delta$ for $n\succsim\log^{\frac1{2\alpha}}(1/\delta)$. Under this event, and taking into account the contraction bound from Lemma~\ref{lem:M_bounds}, we have
    \begin{align*}
        \hat\theta_{t+1}&=m_n(\hat\theta_t)=m(\hat\theta_t)+m_n(\hat\theta_t)-m(\hat\theta_t)\\
        &\le\rho\hat\theta_t+c\theta_0\sqrt{\frac{\log(1/\delta)}{n}}\\
        &\le\rho^2\hat\theta_{t-1}+\rho c\theta_0\sqrt{\frac{\log(1/\delta)}{n}}+c\theta_0\sqrt{\frac{\log(1/\delta)}{n}}\le\ldots\\
        &\le\rho^{t+1}\theta_0+\left(\sum_{s=0}^{t}\rho^s\right)c\theta_0\sqrt{\frac{\log(1/\delta)}{n}}\\
        &\le\rho^{t+1}\theta_0+\frac{c\theta_0}{1-\rho}\sqrt{\frac{\log(1/\delta)}{n}},
    \end{align*}
    where $\rho:=\frac{1+\theta_0^2/d}{(1-\theta_0^2/d)^2}\cdot\left(1-\frac{(2p-1)^2}{2}\right)\in(0,1)$ as long as  $\theta_0$ satisfies the requirement stated in the theorem. Thus for
    $$
    T\ge\frac{1}{\log(1/\rho)}\log\left\{(1-\rho)\sqrt{\frac{n}{\log(1/\delta)}}\right\}
    $$
    we are guaranteed that the iterate $\hat\theta_T$ satisfies
    $$
    \hat\theta_T\le\frac{c\theta_0}{1-\rho}\sqrt{\frac{\log(1/\delta)}{n}}
    $$
    for $n\succsim\log^{\frac1{2\alpha}}\frac1\delta$ with probability at least $1-\delta$
\end{proof}

\subsection{Lower bound for $m(\theta)$}\label{app:lowerbound}
\begin{lemma}\label{lem:M_LB}
    Let $m(\theta)$ be the function defined in \eqref{eq:m_fun}. Then for $\theta\in\left[0,\left(\sqrt2+\frac1{\sqrt2d}\right)^{-1}\right]$, we have
    $$
    m(\theta)\ge4p(1-p)\left(1-\frac{4\theta^2}{(1-\theta^2/d)^2}\right)\theta.
    $$
\end{lemma}

\begin{proof} For $Z\sim\mathcal{N}(0,1)$, by Stein's lemma,
    \begin{align}
        m(\theta)&=\E\left[t_p\left(\frac{\theta Z}{1-\theta^2/d}\right)Z\right]=\frac{1-\theta^2/d}{\theta}\E\left[t_p\left(\frac{\theta Z}{1-\theta^2/d}\right)\frac{\theta Z}{1-\theta^2/d}\right]\notag\\
        &=\frac{(1-\theta^2/d)\theta^2}{\theta(1-\theta^2/d)^2}\E\left[t_p'\left(\frac{\theta Z}{1-\theta^2/d}\right)\right]=\frac{\theta}{1-\theta^2/d}\E\left[\frac{4p(1-p)}{\left\{c_p\left(\frac{\theta Z}{1-\theta^2/d}\right)\right\}^2}\right]\notag\\
        &\stackrel{(\bigstar)}{\ge}\frac{\theta\cdot4p(1-p)}{(1-\theta^2/d)\cdot\E\left[c^2_p\left(\frac{\theta Z}{1-\theta^2/d}\right)\right]}\ge\frac{\theta\cdot4p(1-p)}{\E\left[c^2_p\left(\frac{\theta Z}{1-\theta^2/d}\right)\right]},\label{eq:m_lb_1}
    \end{align}
    where the step $(\bigstar)$ is due to Jensen's inequality. Notice that
    $$
    c^2_p(x)=(p e^{x}+(1-p)e^{-x})^2=2p(1-p)+p^2e^{2x}+(1-p)^2e^{-2x}.
    $$
    Thus,
    \begin{align}
    &\E\left[c_p^2\left(\frac{\theta Z}{1-\theta^2/d}\right)\right]\notag\\
    &=2p(1-p)+p^2\E\left[\exp\left(\frac{2\theta Z}{1-\theta^2/d}\right)\right]+(1-p)^2\E\left[\exp\left(-\frac{2\theta Z}{1-\theta^2/d}\right)\right]\notag\\
    &=2p(1-p)+(p^2+(1-p)^2)\cdot\exp\left(\frac{2\theta^2}{(1-\theta^2/d)^2}\right)\notag\\
    &\stackrel{(\diamondsuit)}{\le}2p(1-p)+(1-2p(1-p))\left(1+\frac{4\theta^2}{(1-\theta^2/d)^2}\right)\notag\\
    &=1+\frac{4(p^2+(1-p)^2)}{(1-\theta^2/d)^2}\theta^2\le1+\frac{4}{(1-\theta^2/d)^2}\theta^2,\label{eq:cp_ub}
    \end{align}
    where the step $(\diamondsuit)$ is valid for $\theta\in\left[0,\frac{1}{\sqrt{2}+\frac{1}{\sqrt2 d}}\right]$ due to an elementary inequality $e^x<1+2x$ for $x\in[0,1]$.
    From \eqref{eq:m_lb_1} and \eqref{eq:cp_ub}, we have
    $$
    m(\theta)\ge\frac{4p(1-p)}{\left(1+\frac{4}{(1-\theta_0^2/d)}\theta^2\right)}\theta\ge4p(1-p)\left(1-\frac{4\theta^2}{(1-\theta^2/d)^2}\right)\theta\quad\text{for }\theta\in\left[0,\frac{1}{\sqrt2+\frac1{\sqrt2d}}\right],$$     where we used the fact that $\frac1{1+x}\ge1-x$ for $x\ge0$.
\end{proof}

\subsection{Proof of Lemma~\ref{lem:excess}}
\begin{proof}
    It is well-known that 
    \begin{equation}
    \Err[h]-\Err[h^\ast]=\E\big[\left|2\eta(\mathbf{X})-1\right|\mathbb{I}\{h(\mathbf{X})\ne h^\ast(\mathbf{X})\}\big]    \label{eq:excess_0}
    \end{equation}
    (see the proof of Theorem 2.1 in \citet{DBLP:books/sp/DevroyeGL96}). Define the sets
    $$
    \mathcal{A}:=\left\{\mathbf{x}\in\mathbb{R}^d:\,2\eta(\mathbf{x})>1\right\}\quad\text{and}\quad\mathcal{B}:=\left\{\mathbf{x}\in\mathbb{R}^d:\,h(\mathbf{x})\ne h^\ast(\mathbf{x})\right\}.
    $$
    For $\mathbf{x}\in\mathcal{A}\cap\mathcal{B}$, we have
    \begin{align}
    &|2\eta(\mathbf{x})-1|f(\mathbf{x})=[2\eta(\mathbf{x})-1]f(\mathbf{x})=[\eta(\mathbf{x})-(1-\eta(\mathbf{x}))]f(\mathbf{x})\notag\\
    &=\underbrace{\eta(\mathbf{x})f(\mathbf{x})}_{\frac12f_1(\mathbf{x})}-\frac12\tilde{f}_1(\mathbf{x})+\frac12\tilde{f}_0(\mathbf{x})-\underbrace{\big[1-\eta(\mathbf{x})\big]f(\mathbf{x})}_{\frac12f_0(\mathbf{x})}+\frac12\big[\underbrace{\tilde{f}_1(\mathbf{x})-\tilde{f}_0(\mathbf{x})}_{\le\,0}\big]\notag\\
    &\le\frac12\left|f_1(\mathbf{x})-\tilde{f}_1(\mathbf{x})\right|+\frac12\left|f_0(\mathbf{x})-\tilde{f}_0(\mathbf{x})\right|,\label{eq:excess_1}
    \end{align}
    where we used the following facts:
    \begin{align*}
    &\eta(\mathbf{x})f(\mathbf{x})=\Pr[Y=1\mid\mathbf{X}=\mathbf{x}]f(\mathbf{x})=f_1({\mathbf{x}})\Pr[Y=1]=\frac12 f_1(\mathbf{x}),\\
    &\big[1-\eta(\mathbf{x})\big]f(\mathbf{x})=\Pr[Y=0\mid\mathbf{X}=\mathbf{x}]f(\mathbf{x})=f_0({\mathbf{x}})\Pr[Y=0]=\frac12 f_0(\mathbf{x}),\\
    &\mathbf{x}\in\mathcal{A}\cap\mathcal{B} \quad\Leftrightarrow \quad h^\ast(\mathbf{x})=1\,\, \land \,\,h(\mathbf{x})=0\quad\Rightarrow\quad\tilde{f}_1(\mathbf{x})\le\tilde{f}_0(\mathbf{x}).
    \end{align*}
    Similarly, for $\mathbf{x}\in\mathcal{A}^c\cap\mathcal{B}$, we get
    \begin{align}
    &|2\eta(\mathbf{x})-1|f(\mathbf{x})=[1-2\eta(\mathbf{x})]f(\mathbf{x})=[(1-\eta(\mathbf{x}))-\eta(\mathbf{x})]f(\mathbf{x})\notag\\
    &=\underbrace{\big[1-\eta(\mathbf{x})\big]f(\mathbf{x})}_{\frac12f_0(\mathbf{x})}-\frac12\tilde{f}_0(\mathbf{x})+ \frac12\tilde{f}_1(\mathbf{x}) -\underbrace{\eta(\mathbf{x})f(\mathbf{x})}_{\frac12f_1(\mathbf{x})}+\frac12\big[\underbrace{\tilde{f}_0(\mathbf{x})-\tilde{f}_1(\mathbf{x})}_{\le\,0}\big]\notag\\
    &\le\frac12\left|f_0(\mathbf{x})-\tilde{f}_0(\mathbf{x})\right|+\frac12\left|f_1(\mathbf{x})-\tilde{f}_1(\mathbf{x})\right|,\label{eq:excess_2}
    \end{align}
    From \eqref{eq:excess_1} and \eqref{eq:excess_2}, we have
    \begin{align*}
    \Err[h]-\Err[h^\ast]&=\int_{\{\mathbf{x}\in\mathbb{R}^d:\,h(\mathbf{x})\ne h^\ast(\mathbf{x})\}}\left|2\eta(\mathbf{x})-1\right|f(\mathbf{x})d\mathbf{x}\\
    &\le\frac12\int_{\mathbb{R}^d}\left|f_0(\mathbf{x})-\tilde{f}_0(\mathbf{x})\right|d\mathbf{x}+\frac12\int_{\mathbb{R}^d}\left|f_1(\mathbf{x})-\tilde{f}_1(\mathbf{x})\right|d\mathbf{x}.
    \end{align*}
\end{proof}

\end{document}